\documentclass[sigconf]{aamas}

\usepackage{balance} 
\usepackage{algorithm}      
\usepackage{algpseudocode} 
\usepackage{xcolor}
\usepackage{subcaption}
\usepackage{dblfloatfix}
\usepackage{booktabs}
\newcommand{\sigmoid}[1]{\text{Sigmoid}\!\left(#1\right)}

\doi{VBFF4869}

\makeatletter
\gdef\@copyrightpermission{
  \begin{minipage}{0.2\columnwidth}
   \href{https://creativecommons.org/licenses/by/4.0/}{\includegraphics[width=0.90\textwidth]{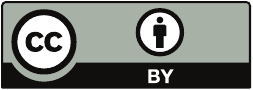}}
  \end{minipage}\hfill
  \begin{minipage}{0.8\columnwidth}
   \href{https://creativecommons.org/licenses/by/4.0/}{This work is licensed under a Creative Commons Attribution International 4.0 License.}
  \end{minipage}
  \vspace{5pt}
}
\makeatother

\setcopyright{ifaamas}
\acmConference[AAMAS '26]{Proc.\@ of the 25th International Conference
on Autonomous Agents and Multiagent Systems (AAMAS 2026)}{May 25 -- 29, 2026}
{Paphos, Cyprus}{C.~Amato, L.~Dennis, V.~Mascardi, J.~Thangarajah (eds.)}
\copyrightyear{2026}
\acmYear{2026}
\acmDOI{}
\acmPrice{}
\acmISBN{}

\title[LRT-Diffusion]{LRT-Diffusion: Calibrated Risk-Aware Guidance for Diffusion Policies}

\thanks{This paper has been accepted to AAMAS 2026.}

\author{Ximan Sun}
\orcid{0009-0004-3254-9935}
\affiliation{
  \institution{Duke University}
  \city{Durham}
  \country{United States}}
\email{ximan.sun@duke.edu}

\author{Xiang Cheng}
\orcid{0009-0001-1458-7821}
\affiliation{
  \institution{Duke University}
  \city{Durham}
  \country{United States}}
\email{xiang.cheng@duke.edu}

\begin{abstract}
Diffusion policies are competitive for offline Reinforcement Learning but are typically guided at sampling time by heuristics that lack a statistical notion of risk. We introduce LRT-Diffusion, a risk-aware sampling rule that performs evidence accumulation between two inference-time heads: an unconditional background head and a state-conditional good head. Concretely, we accumulate a log‑likelihood ratio and gate the conditional mean with a logistic controller whose threshold $\tau$ is calibrated once per task and per sampler under $H_0$ to meet a user‑specified Type‑I level $\alpha$. This turns guidance from a fixed push into an \emph{evidence-driven} adjustment with a user-interpretable risk budget. Importantly, we deliberately leave training vanilla (two heads with standard $\epsilon$-prediction) under the structure of DDPM. LRT guidance composes naturally with Q‑gradients: critic‑gradient updates can be taken at the unconditional mean, at the LRT‑gated mean, or a blend, exposing a continuum from exploitation to conservatism. We standardize states/actions consistently at train and test time and report a state‑conditional OOD metric alongside return. On D4RL MuJoCo tasks, LRT‑Diffusion yields a calibrated return–risk frontier: LRT often reduces state-conditional OOD, and combining with a small Q-step increases return along the frontier. Theoretically, we establish level‑$\alpha$ calibration, stability bounds, and a return comparison showing when evidence-gated guidance is preferable to pure Q-guidance. Overall, LRT‑Diffusion is a drop‑in, inference‑time method that adds principled, calibrated risk control to diffusion policies for offline RL.
\end{abstract}


\newcommand{\BibTeX}{\rm B\kern-.05em{\sc i\kern-.025em b}\kern-.08em\TeX}

\begin{document}


\pagestyle{fancy}
\fancyhead{}


\maketitle 


\section{Introduction}
Offline reinforcement learning (RL) aims to learn high‑performing policies from fixed datasets without further environment interaction. A central difficulty is \emph{distributional shift}: actions proposed by the learned policy can drift away from the behavioral support where value estimates are reliable. Diffusion policies have recently emerged as strong generative decision‑makers for offline RL~\cite{Janner2022Diffuser}: by learning a conditional diffusion model over actions given state, they produce smooth, high‑fidelity samples that respect support better than direct regression. However, \emph{how} these policies are guided at sampling time remains largely heuristic. Common practices—injecting $Q$‑gradients with hand‑tuned schedules and ad‑hoc clipping—lack a statistical notion of risk and offer limited control of the return–shift trade‑off. Unlike prior works, our approach keeps training strictly vanilla (no critic-guided losses) and moves all risk control to inference via a calibrated likelihood-ratio gate (LRT), yielding an interpretable, reproducible risk knob.

\textit{\textbf{Intuition.} }
Rather than always pulling samples toward the conditional head, we ask at every denoising step: \emph{is there enough evidence to move toward the “good‑action” direction for this state?} Concretely, we split offline actions into a \emph{good} subset and a \emph{background} subset, train a two‑head diffusion model (an \emph{unconditional} head on all data and a \emph{conditional} head on good data) with class‑balancing and optional advantage‑based soft weights, and then gate the conditional pull at inference by a \emph{calibrated} likelihood ratio. The gate is motivated by the Neyman–Pearson test~\cite{NeymanPearson1933}: a single user knob Type‑I rate $\alpha$ controls the tolerated false activations under $H_0$ (“background is correct”); while the hard LRT is UMP at level $\alpha$ under equal covariances, we use a smooth gate in practice for numerical stability and keep $\alpha$ interpretable via calibration on held‑out states.

\textit{\textbf{Method overview.}}
We introduce \emph{LRT‑Diffusion}, a risk‑aware, inference‑only sampling scheme for diffusion policies (see Fig~\ref{fig:overview}).
At each denoising step we make a binary decision between a background prediction and a data‑conditioned prediction.
We accumulate evidence during sampling and open the gate only when the evidence is strong, interpolating between the two predictions with a data‑dependent weight.
A single threshold is calibrated once per task so that the empirical false‑activation rate does not exceed a user‑chosen $\alpha$, turning guidance from a fixed push into an \emph{evidence‑driven} adjustment: weak evidence $\Rightarrow$ stay near the background prior; strong evidence $\Rightarrow$ move decisively toward the conditional policy.
Training is unchanged. A detailed, step-by-step illustration of the inference-time pipeline is provided in Fig.~\ref{fig:app-detail}.

\begin{figure*}[t]
  \centering
  \includegraphics[width=\textwidth]{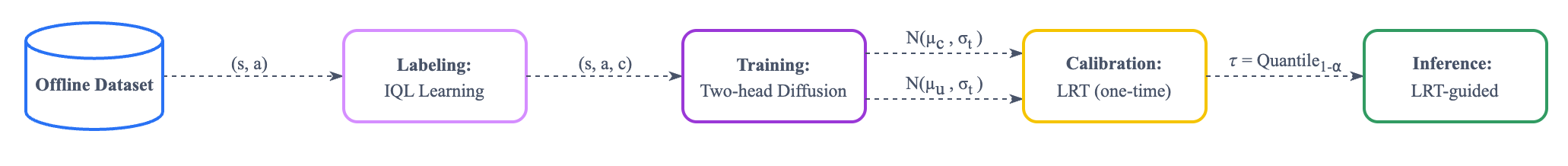}
  \caption{Overall pipeline of LRT-Diffusion. Training is vanilla; all risk control is applied at inference via a calibrated gate.}
  \label{fig:overview}
\end{figure*}

\textit{\textbf{Composition with value guidance.}}
We optionally combine LRT gating with a small critic step: at each denoising update, take a capped step that increases the learned critic at the current proposal (standard decreasing schedule). The gradient can be evaluated near the background proposal, the LRT-gated proposal, or an interpolation; the threshold is calibrated with the same choice, so the Type-I semantics are preserved.

\textit{\textbf{Theoretical guarantees at a glance.}}
Under equal‑covariance heads, the hard gate is the uniformly most powerful test at level $\alpha$; our soft gate (a numerically stable surrogate) retains the threshold’s semantics and improves stability. Beyond testing optimality, we analyze offline‑RL error propagation and show that reducing the policy’s \emph{state‑conditional} OOD rate—which LRT lowers in many regimes—tightens a lower bound on true return relative to pure $Q$‑guidance when the critic is unreliable off‑support.

\textbf{\textit{Practical picture.}}
The gate exposes monotone, interpretable knobs: $\alpha$ (risk), $\beta_{\max}$ (max pull), and $\delta$ (sharpness). LRT can be combined with a small critic step evaluated at the background mean, the LRT-gated mean, or a blend; the threshold is calibrated with that same choice so the Type-I semantics are preserved. Empirically on D4RL MuJoCo, LRT tracks the target $\alpha$ and yields a calibrated return-risk frontier: LRT serves as a low-risk anchor, while LRT+Q increases return along the frontier.

\subsection{Contributions}

\textit{\textbf{Calibrated LRT guidance for diffusion policies.}}
Unlike prior research integrating value or energy guidance into the diffusion model, we cast each denoising step as a likelihood-ratio test between a background head and a ``good'' head, and calibrate a single threshold $\tau$ on held-out states to bound the trajectory-level Type-I rate at a user-chosen $\alpha$. Training remains vanilla (two heads with class balancing and optional advantage weights).

\textit{\textbf{Labeling-and-weighting recipe for the conditional head.}}
Rather than fitting a single conditional diffusion model to all (state, action) pairs, as is common in
return-dominated diffusion or $Q$-guided training, we split offline actions into a \emph{good} subset and a \emph{background} subset using IQL advantages (top-$p$ quantile), and train a two-head diffusion model with class-balancing and advantage-aware soft weights. This improves the quality of the conditional head without extra supervision. \footnote{E.g., Diffuser/Decision‑Diffuser use return conditioning and sampling-time guidance; Diffusion‑QL
injects value signals into the training loss, all without a binary good/background head.}

\textit{\textbf{Actionable theory with finite-sample and stability guarantees.}} Under equal covariances the hard LRT is UMP at level $\alpha$ (Prop.~\ref{prop:np-ump}); a one-shot Monte-Carlo calibration yields
$\mathbb{P}_{H_0}(\ell_{\mathrm{cum}}\ge \hat\tau)\le \alpha+\varepsilon_n$ (Thm.~\ref{thm:dkw}). We bound the per-step deterministic drift and show the cumulative LLR is sub-Gaussian with an explicit variance proxy (Sec.~\ref{sec:theory-calibration-stability}), explaining why calibration is numerically stable.

\textit{\textbf{Return bounds that connect risk control to distribution shift.}}
With a standard offline error split, we prove a comparison bound for true returns between LRT-guided and $Q$-guided sampling (Prop.~\ref{prop:gap}). A mild monotonicity assumption turns the calibrated level~$\alpha$ into a conservative OOD upper bound (Prop.~\ref{prop:eta}), clarifying when and why LRT is preferable to pure $Q$-guidance.

\textit{\textbf{Compatibility with value gradients via matched calibration.}}
LRT controls \emph{risk} while the $Q$‑step pursues \emph{return}, and these roles are orthogonal. We show that composing LRT with a small $Q$-step preserves level-$\alpha$ semantics \emph{so long as} calibration uses the exact deployed sampler (Prop.~\ref{prop:calib-same-sampler}). Thus, risk control (through $\alpha$ and its calibrated $\tau$) is decoupled from return‑seeking (through $\nabla_a\hat Q$), making the $Q$‑update a plug‑and‑play module. The knobs $(\alpha,\beta_{\max},\delta)$ remain monotone and interpretable.

\textit{\textbf{Empirical validation.}}
On D4RL MuJoCo, LRT-Diffusion honors the target Type-I rate and yields a calibrated return-risk frontier, with ablations over $\alpha$, $\beta_{\max}$, $\delta$, and gradient anchoring.


\section{Related Work}

\textit{\textbf{Distribution shift in offline RL and conservative learning.}}
A central obstacle in offline RL is that policy actions may fall outside the behavior support, where critic estimates are brittle. Conservative algorithms explicitly counter this by penalizing or constraining OOD actions, e.g. ~\citep{Fujimoto2019BCQ,Kumar2019BEAR,Kumar2020CQL,Fujimoto2021TD3BC,Kostrikov2022IQL} learns values via expectile regression and performs advantage-weighted improvement without explicit behavior cloning. Our work is complementary: we keep training unchanged (e.g., use IQL advantages only to label/weight “good” actions) and control OOD at \emph{sampling time} via a calibrated gate.

\textit{\textbf{Diffusion policies for decision-making and their guidance.}}
Diffusion models have been adopted for control and planning ~\cite{Janner2022Diffuser}. To improve returns, many systems apply \emph{Q-guidance}—adding an action-space step along $\nabla_a \hat Q(s,a)$ with hand-tuned schedules and clips—conceptually analogous to guidance in image diffusion ~\cite{Dhariwal2021GuidedDiffusion,Ho2022CFG}. However, such heuristics lack a statistical notion of risk: there is no global control on the probability of “falsely” pulling away from the background prior.

\textit{\textbf{Risk-aware control and calibration.}}
Risk-sensitive and conservative offline RL methods (e.g., behavior-regularized or pessimistic objectives)
encode risk \emph{in the training loss} by penalizing value estimates or constraining policy deviation from the dataset
\citep{Fujimoto2019BCQ,Wu2019BRAC,Kumar2020CQL,Nair2020AWAC,Peng2019AWR,Kostrikov2022IQL}.
A complementary line uses \emph{distribution-free calibration} to turn data-driven thresholds into finite-sample guarantees
(e.g., conformal prediction) \citep{Vovk2005ALRW,Angelopoulos2023ConformalSurvey}.
In diffusion-based decision making, sampling is typically guided heuristically (e.g., $Q$-gradient pushes with hand-tuned schedules/clipping) without an explicit statistical notion of risk.
Our approach bridges these threads: we cast each reverse step as a simple-vs-simple test between an unconditional (background) head and a conditional (good) head, accumulate a log-likelihood ratio (LLR), and \emph{calibrate one threshold} on held-out states so that the empirical Type-I rate under $H_0$ does not exceed a user-chosen $\alpha$.
Under equal covariances—satisfied by our two-head design—the hard likelihood-ratio test is uniformly most powerful \citep{NeymanPearson1933}, and we use a smooth gate for numerical stability while preserving the same $\alpha$-semantics via calibration.
This gives a statistically grounded alternative to heuristic mixing: guidance becomes evidence-driven with an explicit level-$\alpha$ risk budget at inference time, without changing training.

\textit{\textbf{Hypothesis testing and likelihood ratios in RL}}
Testing in RL is typically used for deployment-time decisions (high-confidence OPE~\cite{Thomas2015HCOPE}, safe improvement such as SPIBB~\cite{Laroche2019SPIBB}), rather than within a generative sampler.
Safety-constrained methods (e.g., CPO~\cite{Achiam2017CPO}) control violations but do not apply likelihood-ratio gating during action generation.
Although sequential tests like SPRT~\cite{Wald1945SPRT} inspire sequential criteria, we are not aware of calibrated likelihood-ratio gates used to steer diffusion denoising at inference.
Our approach turns guidance into an evidence-controlled, level-$\alpha$ procedure.


\textit{\textbf{Advantage-based labeling and weighting.}}
For higher return actions, advantage-weighted and critic-regularized schemes (~\citep{Peters2007RWR,Peng2019AWR,Nair2020AWAC}) are proposed. We use IQL advantages to (i) define a top-$p$ “good” subset for the conditional head and (ii) optionally apply a temperature-controlled soft weight on positives, while keeping a background head trained on all data. This two-head setup strengthens the conditional signal without altering the base training pipeline, and remains compatible with our calibrated LRT at inference. Our two-head design enforces equal covariances by construction, which simplifies the LLR and aligns with the NP test.


\section{Background}
Offline RL is given a fixed dataset $\mathcal{D}=\{(s^{(i)},a^{(i)},r^{(i)},s^{(i)'})\}_{i=1}^N$, collected by an unknown behavior policy \(\pi\), and aims to learn a high-return policy while controlling out-of-distribution risk. In the dataset, \(s\in\mathbb R^{d_s}\) is the current state, \(a\in\mathbb R^{d_a}\) is the action, \(r\in\mathbb R\) is the reward, \(s'\) is the next state. We treat the process as a discounted Markov Decision Process with continuous state/action spaces.

\subsection{IQL Advantages and “Good vs.\ Background” Labels}

For each pair of state and action $(s,a)\in \mathcal{D}$, we train an IQL-style critic~\cite{Kostrikov2022IQL} \((\hat Q(s,a),\hat V(s))\) on standardized inputs and define an \emph{advantage} $A(s,a)$ as $A(s,a) \;=\; \hat Q(s,a) - \hat V(s)$,
which quantifies the value contributed by action $a$ at state $s$ relative to the state’s baseline. Let \(\kappa\) be the top \((p)\)-quantile of \(\{A(s_i,a_i)\}_{i=1}^N\).
We label each pair of state-action data by $c \;=\; \mathbf{1}\{A(s,a)\ge \kappa\}\in\{0,1\}$.

Pairs with \(c{=}1\) are considered \emph{good}, as their advantages fall within the top \((p)\) fraction of the dataset, while those with \(c{=}0\) form the non-advantage subset. In practice, we set \(p=0.2\), allocating approximately \(20\%\) of the data as \emph{good} examples.

\subsection{Diffusion Policies with Two Heads}
\label{sec:two-heads}
We train a diffusion policy with two heads: an \emph{unconditional} head
trained on all state–action pairs \(\mathcal{D}\), and a \emph{conditional}
head trained only on the high-advantage subset we labeled beforehand.
Formally, let \(\mathcal{D}_{\mathrm{all}} := \mathcal{D}\) and
\(\mathcal{D}_{\mathrm{good}} := \{(s,a,r,s') \in \mathcal{D} : c=1\}\).
The unconditional head learns broad dataset coverage from
\(\mathcal{D}_{\mathrm{all}}\), while the conditional head specializes to
advantage behavior using \(\mathcal{D}_{\mathrm{good}}\).

\textit{\textbf{Two heads and notation.}}
We use a shared backbone \(\phi_\theta(s,a_t,t)\) with two output branches (“heads”):
an \emph{unconditional/background} head trained on
\(\mathcal{D}_{\mathrm{all}}\), and a \emph{conditional/good} head 
trained on \(\mathcal{D}_{\mathrm{good}}\). Each head maps features to a DDPM-style
noise prediction,
\begin{equation}
\label{eq:diffusion-epsilon}
\hat\varepsilon_{u} \;=\; h_{u,\theta}\!\big(\phi_\theta(s,a_t,t)\big),
\qquad
\hat\varepsilon_{c} \;=\; h_{c,\theta}\!\big(\phi_\theta(s,a_t,t)\big),
\end{equation}
where the subscripts \(u\) and \(c\) will consistently denote
\emph{unconditional} and \emph{conditional} quantities, respectively.

\textit{\textbf{DDPM parameterization.}}
We adopt the predict-\(\varepsilon\) parameterization \cite{Ho2020DDPM,Nichol2021ImprovedDDPM}
with schedule \(\{\alpha_t,\bar\alpha_t\}_{t=1}^T\), where
\(\bar\alpha_t=\prod_{\tau=1}^t \alpha_\tau\).
Given the current latent action \(a_t\) at step \(t\), a noise prediction
\(\hat\varepsilon\) induces a Gaussian reverse kernel with shared covariance
\(\sigma_t^2 I\) and mean
\[
\mu(\hat\varepsilon; t,s,a_t)
~=~
\frac{1}{\sqrt{\alpha_t}}
\!\left(
a_t \;-\; \frac{1-\alpha_t}{\sqrt{1-\bar\alpha_t}}\;\hat\varepsilon
\right).
\]
Throughout this subsection, we assume that the two diffusion heads share the same reverse-time variance $\sigma_t^2$, as given by the diffusion schedule, which enables a closed-form expression for the step-wise log-likelihood ratio. We discuss this equal-variance choice and examine robustness when it is deliberately violated at inference time in Sec.~\ref{sec:Ablation}. We define the head-specific means
\[
\mu_u(t,s,a_t):=\mu(\hat\varepsilon_u; t,s,a_t),
\qquad
\mu_c(t,s,a_t):=\mu(\hat\varepsilon_c; t,s,a_t),
\]
and use the same \(\sigma_t>0\) from the diffusion schedule for both heads. This yields two reverse kernels for the next-step action:
\[\tiny
a_{t-1}\mid a_t,s ~\sim~ \mathcal{N}\!\big(\mu_u(t,s,a_t),\,\sigma_t^2 I\big),
\quad
a_{t-1}\mid a_t,s ~\sim~ \mathcal{N}\!\big(\mu_c(t,s,a_t),\,\sigma_t^2 I\big),
\]
where $\mathcal{N}(\mu,\Sigma)$ is Gaussian with mean $\mu$ and variance $\Sigma$.

After training, the parameters \(\theta\) are frozen, so
\(\mu_u(t,s,a_t)\) and \(\mu_c(t,s,a_t)\) are deterministic functions of
\((t,s,a_t)\) (with the same \(\sigma_t\)). We will interpret a single reverse step through these two proposals in the next subsection.

\textit{\textbf{Notation clarification.}}
 We write dataset indices as parenthesized superscripts, e.g., \(a^{(i)}\) is the $i-th$ action in the dataset $\mathcal{D}$.
We reserve subscripts \(t\) for diffusion steps, e.g., \(a_t\) is the latent at step \(t\) and \(a_0\) is the final action..

\subsection{Hypotheses for a Reverse Step}
\label{sec:hypotheses-reverse}
With the trained parameters frozen, each reverse step $t$ and conditioning state–latent
pair $(s,a_t)$ yields two Gaussian proposals for the next latent action $a_{t-1}$:
\begin{align*}
H_0~(\text{background}):~~ a_{t-1}\mid a_t,s ~\sim~
\mathcal{N}\!\big(\mu_u(t,s,a_t),\,\sigma_t^2 I\big),
\\
H_1~(\text{good}):~~ a_{t-1}\mid a_t,s ~\sim~
\mathcal{N}\!\big(\mu_c(t,s,a_t),\,\sigma_t^2 I\big),
\end{align*}

where $\mu_u$ and $\mu_c$ are the head-specific means defined in the previous subsection~\ref{sec:two-heads}. Intuitively, $H_0$ favors
broad, background behavior supported by the entire dataset, while $H_1$ emphasizes
high-advantage behavior learned from the $c{=}1$ subset.

For compactness, denote the head-induced one-step densities by
\begin{align}
p_c(a_{t-1}\mid a_t,s,t)
:= & \mathcal{N}\!\big(a_{t-1};\,\mu_c(t,s,a_t),\,\sigma_t^2 I\big),\\
p_u(a_{t-1}\mid a_t,s,t)
:= & \mathcal{N}\!\big(a_{t-1};\,\mu_u(t,s,a_t),\,\sigma_t^2 I\big).
\end{align}

\subsection{Likelihood-Ratio View of a Reverse Step}
\label{sec:lrt-reverse}

\textit{\textbf{Unknown labels $\Rightarrow$ a trajectory-level test.}}
At inference the label $c$ is unknown. Given two trajectory models
$p_u(a_{T:1}\mid s)$ (background) and $p_c(a_{T:1}\mid s)$ (good), the
Neyman--Pearson (NP) lemma implies that the level-$\alpha$ test that rejects
$H_0$ when the cumulative log–likelihood ratio (LLR) exceeds a threshold
is uniformly most powerful (UMP). We thus base our gate on the
trajectory statistic
\[\ell_{\mathrm{cum}}(a_{T:1})
\;:=\; \log\frac{p_c(a_{T:1}\mid s)}{p_u(a_{T:1}\mid s)}.\]

\textit{\textbf{How to compute $\ell_{\mathrm{cum}}$: step-wise LLRs.}}
With the reverse chain factorization in DDPM, for $i\in\{u,c\}$
\[p_i(a_{T:1}\mid s)=q(a_T)\prod_{t=1}^{T} p_i(a_{t-1}\mid a_t,s,t),\]
so the trajectory LLR decomposes into a sum of step-wise terms:
\begin{equation}
\label{eq:cum-llr}
\ell_{\mathrm{cum}}(a_{T:1})
~=~ \sum_{t=1}^{T}\ell_t,
\qquad
\ell_t
~:=~ \log\frac{p_c(a_{t-1}\mid a_t,s,t)}{p_u(a_{t-1}\mid a_t,s,t)}.
\end{equation}

\textit{\textbf{Head-induced one-step densities and quadratic form.}}
From Sec.~\ref{sec:hypotheses-reverse}, each head induces a Gaussian kernel with
shared covariance. Hence the step-wise LLR admits the shared-variance simplification
\begin{equation}
\label{eq:llr-quadratic}
\ell_t
~=~ \frac{1}{2\sigma_t^{2}}\Big(\|a_{t-1}-\mu_u(t,s,a_t)\|_2^2-\|a_{t-1}-\mu_c(t,s,a_t)\|_2^2\Big),
\end{equation}
i.e., a linear discriminant in $a_{t-1}$ and cheap to evaluate.

\textit{\textbf{Decision rule for assigning labels.}}
The NP test uses $\ell_{\mathrm{cum}}$ with a calibrated threshold $\tau$ to control
the level-$\alpha$ Type-I rate. We implement a smooth gate that monotonically
approximates the NP decision and calibrate $\tau$ under $H_0$; details are given in
Sec.~\ref{sec:method}.


\section{Method: LRT-Diffusion}
\label{sec:method}
We split offline actions by an advantage threshold into a \emph{good} subset and a
\emph{background} subset. We then train a \emph{vanilla} two‑head diffusion policy
(DDPM $\epsilon$‑prediction \cite{Ho2020DDPM}): an unconditional head on all data and a conditional head on the good subset, with class‑balancing and optional advantage‑based soft weights; no value/energy‑guided losses are added. At inference, each reverse step is treated as a binary test between heads:
we accumulate a cumulative log‑likelihood ratio (LLR; Sec.~\ref{sec:lrt-reverse}) and apply a calibrated logistic gate to interpolate the mean,
\begin{eqnarray}
\mu_{\text{LRT},t}\;=\;\mu_u(t,s,a_t)\;+\;\beta_t\,\big(\mu_c(t,s,a_t)-\mu_u(t,s,a_t)\big),
\end{eqnarray}
then sample \(a_{t-1}\sim\mathcal N(\mu_{\text{LRT},t},\sigma_t^2 I)\). The threshold \(\tau\) is calibrated under \(H_0\) (background) to meet a user‑chosen
Type‑I rate \(\alpha\).Below is the end-to-end pipeline:
\begin{enumerate}
    \item Train an IQL critic \((\hat Q,\hat V)\) on standardized inputs;
    \item Compute advantages and label top-\(p\) pairs, see \S\ref{sec:labeling};
    \item Train a two-head diffusion model on \(\epsilon\)-prediction, see \S\ref{sec:diffusion-training};
    \item Calibrate a single threshold \(\tau\) under \(H_0\), see \S\ref{sec:method-calibration};
    \item At inference, accumulate a cumulative LLR and gate the conditional pull, see \S\ref{sec:method-inference};
\end{enumerate}

\subsection{Good‑vs‑background Labeling}
\label{sec:labeling}
We rank state–action pairs by the IQL advantage $A(s,a)=\hat Q(s,a)-\hat V(s)$ computed on standardized inputs, and mark the top‑p global quantile as “good.”
Although one could define state‑wise thresholds, the advantage already subtracts a state‑dependent baseline $\hat V(s)$, making A comparable across states in practice. Moreover, global top‑p is a non‑parametric surrogate of advantage‑weighted learning (e.g., AWR/AWAC), corresponding to a small‑temperature limit without tuning an extra temperature. Practically, offline datasets have highly uneven coverage—many states appear once or with only a handful of actions—so per‑state ranking is statistically brittle; a single global threshold is more stable and reproducible. Crucially, our risk control via LRT is orthogonal to this choice: once the two‑head model is trained, Type‑I error is calibrated at inference regardless of how the conditional head’s subset was selected.

\subsection{Two-head Training}
\label{sec:diffusion-training}
Under the two-head diffusion policy in Sec.~\ref{sec:two-heads}, we optimize the $\epsilon$-prediction loss (refer to Eq.~\ref{eq:diffusion-epsilon}) with per-sample weights
\[\mathcal{L}=\tfrac{1}{B}\sum_{i=1}^B \tilde w_i\,
\big\|\hat\epsilon_\phi(a_{t,i},t,\tilde s_i)-\epsilon_i\big\|_2^2,
\quad \tilde w_i=\frac{w^{\text{cb}}(c_i;\hat\rho)\cdot u_i}{\frac{1}{B}\sum_{j=1}^B w^{\text{cb}}(c_j;\hat\rho)\cdot u_j},
\]
where $w^{\text{cb}}(c;\hat\rho)$ is the \emph{class-balanced} factor, balancing positive/negative contributions within each batch using an EMA estimate; $u_i$ is the optional \emph{within-positive} soft weight that emphasizes stronger positives to $H_1$.
\begin{equation}
\label{eq:weights-indicator}
w^{\mathrm{cb}}(c;\hat\rho)=\frac{\mathbf{1}\{c=1\}}{2\hat\rho+\varepsilon}
+\frac{\mathbf{1}\{c=0\}}{2(1-\hat\rho)+\varepsilon},
\end{equation}
\begin{equation}
u_i=1+\mathbf{1}\{c_i=1\}\,
\min\!\Big\{\max\!\big(0,\tfrac{A_i-\kappa}{\tau_A}\big),\,u_{\max}-1\Big\}.
\end{equation}

Multiplying $w^{\text{cb}}$ and $u_i$ separates roles: the former fixes class imbalance between conditional and unconditional datasets, while the latter reallocates mass \emph{within} $c{=}1$ without changing the effective class ratio; the batch normalization of $\tilde w_i$
stabilizes the step size.

\subsection{Evidence-gated Sampler}

\subsubsection{Motivation}\mbox{}\\
\textit{\textbf{1D evidence direction.}}
From Eq.~\ref{eq:llr-quadratic}, under equal covariances the one-step LLR has
constant gradient $\nabla_{a_{t-1}}\ell_t=\Sigma_t^{-1}(\mu_c-\mu_u)$; among unit directions
$u$ the directional derivative $\langle u,\nabla\ell_t\rangle$ is maximized by
$u\parallel\Sigma_t^{-1}(\mu_c-\mu_u)$ (isotropic: $u\parallel\mu_c-\mu_u$).
We therefore restrict the reverse-step mean to the 1D ray
$\mu_u+\beta(\mu_c-\mu_u),$ and let the scalar gate $\beta$ depend monotonically on the cumulative evidence $\ell_{\mathrm{cum}}$.

\textit{\textbf{Why not always use $\mu_c$?}}
Always setting $a_{t-1}{=}\mu_c{+}\sigma_t z_t$ is equivalent to \emph{always accepting $H_1$}, which (i) removes any control on false activations of the conditional head
and thus increases off-support mass in offline RL, and (ii) couples performance to critic/label
errors where the model is most brittle. Our gate opens \emph{only when there is sufficient
evidence} (large $\ell_{\mathrm{cum}}$), yielding a calibrated budget on false activations
(Prop.~\ref{prop:calib-same-sampler} and Thm.~\ref{thm:dkw}).
Section~\ref{sec:theory-ood-return} further shows that, when off-support errors dominate
($\varepsilon_{\rm out}\!\gg\!\varepsilon_{\rm in}$), reducing the state-conditional OOD rate via
a smaller $\alpha$ tightens a lower bound on the true return (Prop.~\ref{prop:gap}).

\subsubsection{Gate Selection}\mbox{}\\
\label{sec:gates}
\textit{\textbf{Hard gate (UMP at level-$\alpha$).}}
Let $\tau$ be a threshold. The \emph{hard} likelihood-ratio test accepts $H_1$ when
$\ell_{\mathrm{cum}}\ge\tau$ and rejects otherwise. Equivalently, with a cap
$\beta_{\max}\!\in[0,1]$,
$$\beta_t^{\text{h}} \;=\; \beta_{\max}\,\mathbf{1}\{\ell_{\mathrm{cum}}\ge \tau\},\quad
a_{t-1} \;=\; \mu_u \,+\, \beta_t^{\text{h}}(\mu_c-\mu_u) \,+\, \sigma_t z_t.
$$
Under equal covariances and the joint factorization of the reverse chain, this hard test
is uniformly most powerful (UMP) among all level-$\alpha$ tests
(Prop.~\ref{prop:np-ump}).

\textit{\textbf{Soft gate (stable surrogate) and its hard-limit.}}
For numerical stability we use a logistic surrogate
$$\beta_t \;=\; \beta_{\max}\,\sigmoid{\!\tfrac{\ell_{\mathrm{cum}}-\tau}{\delta}},\quad
a_{t-1} \;=\; \mu_u \,+\, \beta_t(\mu_c-\mu_u) \,+\, \sigma_t z_t,
$$
where $z_t\!\sim\!\mathcal{N}(0,I)$ and $\delta{>}0$ controls sharpness (smaller $\delta$ $\to$ sharper switch).
As $\delta\!\downarrow\!0$, the soft rule converges to the hard gate both pointwise and in
trajectory law (Lemma~\ref{lem:soft-hard}). Thus, the $\alpha, \tau, \beta_t$ can be seen as monotone, interpretable risk knobs. The resulting sampler is summarized in Alg.~\ref{alg:lrt}. 

\subsection{Calibration}
\label{sec:method-calibration}
Fix $(\beta_{\max},\delta)$ and freeze the sampler in Alg.\ref{alg:lrt}. On a held-out state set matched to deployment, simulate reverse chains \emph{under $H_0$} using the same
sampler and collect the realized $\ell_{\mathrm{cum}}$; set
$\hat\tau=\mathrm{Quantile}_{1-\alpha}\big\{\ell_{\mathrm{cum}}^{(i)}\big\}$.
Calibrating with the \emph{exact} deployment sampler preserves the level-$\alpha$
semantics (Prop.~\ref{prop:calib-same-sampler}). By the Dvoretzky--Kiefer--Wolfowitz bound~\cite{Dvoretzky1956DKW,Massart1990},
with probability $\ge 1-\eta$ over $n$ calibration draws,
\[\mathbb{P}_{H_0}^{\text{(sampler)}}\!\big(\ell_{\mathrm{cum}}\ge\hat\tau\big)
\;\le\; \alpha+\sqrt{\tfrac{1}{2n}\log\tfrac{2}{\eta}},\]
giving a finite-sample guarantee (Thm.~\ref{thm:dkw}); see Alg.~\ref{alg:calib}.

\subsection{Inference}
\label{sec:method-inference}
We calibrate $\tau$ once per \emph{(task, model, gate hyperparameters)} on a held-out state set, using the same frozen sampler as in Alg.~\ref{alg:lrt} (and Alg.~\ref{alg:calib} for the Monte-Carlo procedure). Thereafter, all rollouts simply call Alg.~\ref{alg:lrt} with this fixed $\hat\tau$; no re-calibration is needed unless hyperparameters change. The final $a_0$ is un-standardized and clipped before deployment (Alg.~\ref{alg:lrt}).

\begin{algorithm}[t]
\caption{Calibration under $H_0$ to obtain $\tau$ (pre-inference sampling)}
\label{alg:calib}
\begin{algorithmic}[1]
\Require Frozen policy $\theta$, steps $T$, risk $\alpha$, gate $(\beta_{\max},\delta)$, noises $\{\sigma_t\}$, states $\mathcal S_{\rm cal}$, budget $n$
\Ensure Threshold $\tau$
\State $\tau\gets+\infty$
\For{$k=1..K_{\max}$}
  \State $\mathcal L\gets\emptyset$
  \While{$|\mathcal L|<n$}
    \State Sample $s\sim \mathrm{Unif}(\mathcal S_{\rm cal}),a_T\sim\mathcal N(0,I)$; $\ell\gets0$
    \For{$t=T..1$}
      \State $(\mu_u,\mu_c)\gets\mu_\theta(t, s, a_t)$
      \State $\beta_{t}\gets\beta_{\max}\sigmoid{(\ell-\tau)/\delta}$(soft-gate)
      \State $a_{t-1}\sim\mathcal N(\mu_u+\beta_t(\mu_c-\mu_u),\sigma_t^2I)$
      \State $\ell\gets\ell-\frac{\|a_{t-1}-\mu_c\|^2-\|a_{t-1}-\mu_u\|^2}{2\sigma_t^2}$
    \EndFor
    \State $\mathcal L\gets\mathcal L\cup\{\ell\}$            
  \EndWhile
  \State $\tau\gets{\rm Quantile}_{1-\alpha}(\mathcal L)$
\EndFor
\State \Return $\tau$
\end{algorithmic}
\end{algorithm}

\begin{algorithm}[t]
\caption{LRT-Guided Inference}
\label{alg:lrt}
\begin{algorithmic}[1]
\Require State $s$, steps $T$, $\tau$, $(\beta_{\max},\delta)$, noises $\{\sigma_t\}$
\Ensure Action $a_0$
\State $a_T\sim\mathcal N(0,I)$; $\ell\gets0$
\For{$t=T..1$}
  \State $(\mu_u,\mu_c)\gets\mu_\theta(t,s,a_t)$
  \State $\beta_t\gets\beta_{\max}\sigmoid{(\ell-\tau)/\delta}$
  \State $a_{t-1}\sim\mathcal N(\mu_u+\beta_t(\mu_c-\mu_u),\sigma_t^2I)$
  \State $\ell\gets\ell-\frac{\|a_{t-1}-\mu_c\|^2-\|a_{t-1}-\mu_u\|^2}{2\sigma_t^2}$
  \State $a_t\gets a_{t-1}$ (optionally use Eq.~\ref{eq:with-q-guide})
\EndFor
\State \Return $a_0$
\end{algorithmic}
\end{algorithm}

\subsection{Composition with Value Gradients}
\label{sec:method-qcompose}
To cleanly separate \emph{risk control} from \emph{return seeking}, we optionally apply a value-guidance step. Specifically, we ascend the \emph{action-space} gradient of a learned critic with a simple schedule and clipping:
\begin{equation}\label{eq:with-q-guide}
a_{t-1}\ \leftarrow\ a_{t-1}\ +\ \lambda_t\,\sigma_t^2\,\nabla_a \hat Q_\theta(s,a)\big|_{a=a_c},
\end{equation}
where $a_c\in\{\mu_u,\ \mu_{\text{LRT},t},\ (1-\rho)\mu_u+\rho\,\mu_{\text{LRT},t}\},$ and $\mu_{\text{LRT},t}=\mu_u+\beta_t(\mu_c-\mu_u)$ is the LRT-gated mean. We use a light, hand-tuned schedule (e.g., $\lambda_t\propto\sigma_t$) and gradient clipping to keep updates stable. 

We use the evaluation center $a_c = \mu_{\text{LRT},t}$ as evidence opens the gate, suggesting the gradient is closer to the good direction. For more evaluation center choice discussion, refer Appendix~\ref{app:evaluation-center}

\section{Theory and Properties}
\label{sec:theory}

\subsection{Gate Selection: UMP vs. Stability}
In this part, we discuss the properties of hard and soft gates. For the full mathematical proof, see Appendix~\ref{app:gate}.

\label{sec:gate}
\begin{proposition}[Neyman--Pearson optimality]
\label{prop:np-ump}
For the reverse chain conditioned on $s$ with two simple hypotheses
follows the two-head reverse model of Sec.~\ref{sec:lrt-reverse}.
Then the Neyman-Pearson test that rejects $H_0$ when $\ell_{\mathrm{cum}}\ge\tau$ is uniformly most powerful
among all level-$\alpha$ tests.
\end{proposition}

\begin{lemma}[Soft$\to$hard limit under logistic gate]
\label{lem:soft-hard}
Assume the two-head reverse model of Sec.~\ref{sec:lrt-reverse} and a \emph{soft} gate denoted in Sec.~\ref{sec:gates}.
At step $t$, the proposal mean is
$\mu_t^{\mathrm{soft}}=\mu_u+\beta_t(\mu_c-\mu_u)$ and the reverse variance is $\sigma_t^2 I$.
Fix $(\tau,\beta_{\max})$. As $\delta\to 0$, we have pointwise
\[\mu_t^{\mathrm{soft}} \;\to\;
\mu_t^{\mathrm{hard}}
:= \mu_u + \beta_{\max}\,\mathbf{1}\{\ell_{\mathrm{cum}}\ge\tau\}\,(\mu_c-\mu_u),\]
and the \emph{trajectory law} induced by the soft-gated sampler converges weakly
to that of the hard-gated sampler.
\end{lemma}

\begin{proposition}[Calibrated semantics under the deployment sampler]
\label{prop:calib-same-sampler}
Fix sampler hyperparameters $(\beta_{\max},\delta)$ (and any deterministic $Q$-composition).
Let $\hat\tau$ be the empirical $(1-\alpha)$ quantile of $\ell_{\mathrm{cum}}$ computed from
i.i.d.\ rollouts under $H_0$ \emph{with the same frozen sampler}.
Then the realized false-activation rate at deployment satisfies
\[\mathbb{P}_{H_0}^{(\text{sampler})}\!\big(\ell_{\mathrm{cum}}\ge \hat\tau\big)
\approx \alpha,\]
up to finite-sample fluctuations (see Thm.~\ref{thm:dkw} for a DKW bound).
\end{proposition}


\subsection{Finite-sample Calibration and Stability}
\label{sec:theory-calibration-stability}
Let $F_0$ be the CDF of the cumulative LLR $\ell_{\mathrm{cum}}$ under $H_0$ for the \emph{frozen} sampler
(soft gate and, if present, the fixed $Q$-step). Given i.i.d.\ calibration draws
$\ell^{(1)},\dots,\ell^{(n)}$ and the empirical CDF $\widehat F_n$, define the plug-in quantile
$\hat\tau=\inf\{x:\widehat F_n(x)\ge 1-\alpha\}$.

\subsubsection{Finite-sample Guarantee}
\begin{theorem}[Calibration accuracy via DKW~\cite{Dvoretzky1956DKW,Massart1990}]
\label{thm:dkw}
For any $\zeta\in(0,1)$, with probability at least $1-\zeta$ over the calibration sample,
\[\mathbb{P}_{H_0}^{\text{(sampler)}}\!\big(\ell_{\mathrm{cum}}\ge \hat\tau\big)
~\le~
\alpha + \varepsilon_n,\qquad
\varepsilon_n=\sqrt{\tfrac{1}{2n}\log\tfrac{2}{\zeta}}.\]
\end{theorem}

Rearranging the bound gives a one-line rule: to guarantee\\
$\mathbb{P}_{H_0}^{(\text{sampler})}(\ell_{\mathrm{cum}}\!\ge\!\hat\tau)\le \alpha+\varepsilon$
with confidence at least $1-\zeta$, it suffices to set sample size $n \;\ge\; \frac{1}{2\varepsilon^2}\,\log\!\frac{2}{\zeta}.$ See Appendix~\ref{app:proof-dkw} for proof.

\subsubsection{Stability}

\begin{lemma}[Deterministic displacement bound]\label{lemma:deterministic-bound}    
At reverse step $t$, write $\Delta\mu_t := \mu_{c,t}-\mu_{u,t}$ and let $g_t$ be the \emph{clipped}
critic gradient with $\|g_t\|\le G$.
Our update has deterministic mean
\[
m_t \;=\; \mu_{u,t} \;+\; \beta_t\,\Delta\mu_t \;+\; \lambda_t\,\sigma_t^2\,g_t,
\qquad
0\le\beta_t\le\beta_{\max},\ \ 0\le\lambda_t\le\lambda_{\max}.
\]
Hence the per-step \emph{deterministic displacement} from the background anchor is bounded by
\begin{equation}
\label{eq:step-bound}
\|m_t-\mu_{u,t}\|
~\le~
\beta_{\max}\,\|\Delta\mu_t\| \;+\; \lambda_{\max}\,\sigma_t^2\,G
\;=:\; B_t.
\end{equation}
If, additionally, $\|\Delta\mu_t\|\le D$ (e.g. clamp on
$\Delta\mu_t$) and $\sigma_t^2\le S^2$, then
$$\|m_t-\mu_{u,t}\|\le \beta_{\max}D+\lambda_{\max}S^2 G \;=:\; B_{\text{step}}$$ for all $t$,
and the cumulative \emph{deterministic} deviation from the background chain across $T$ steps is at most
$\sum_{t=1}^T B_t \le T\,B_{\text{step}}$.
\end{lemma}
\noindent
Thus, the deterministic component of the sampler admits a uniform, hyperparameter‑controlled drift bound both per step and over the full reverse trajectory; see Appendix~\ref{app:det-shift} for the detailed proof.

\textit{\textbf{Implication for one-step LLR fluctuations.}}
Under equal covariances, the one-step LLR increment admits the identity
\begin{align*}
\Delta\ell_t
~=~ &
\underbrace{\frac{\Delta\mu_t^\top\big(m_t-\tfrac{\mu_{c,t}+\mu_{u,t}}{2}\big)}{\sigma_t^2}}_{\text{deterministic part}}
\;+\;
\underbrace{\frac{\Delta\mu_t^\top z_t}{\sigma_t}}_{\text{zero-mean Gaussian}},
\end{align*}
where $z_t\sim\mathcal N(0,I)$. Therefore, conditionally on $(m_t,\mu_{u,t},\mu_{c,t})$,
\begin{align*}
\big|\mathbb E[\Delta\ell_t]\big|
&~\le~ \frac{\|\Delta\mu_t\|}{\sigma_t^2}\,\Big(B_t+\tfrac{1}{2}\|\Delta\mu_t\|\Big),\\
\mathrm{Var}(\Delta\ell_t)
&=~ \frac{\|\Delta\mu_t\|^2}{\sigma_t^2}.
\end{align*}
In particular, if $\|\Delta\mu_t\|\le D$ and is the fixed schedule $\sigma_t\ge\sigma_{\min}>0$, then
$\mathrm{Var}(\Delta\ell_t)\le D^2/\sigma_{\min}^2$ and
$\big|\mathbb E[\Delta\ell_t]\big|\le \frac{D}{\sigma_{\min}^2}\big(B_{\text{step}}+\tfrac{D}{2}\big)$, so
$\ell_{\mathrm{cum}}=\sum_t\Delta\ell_t$ is sub-Gaussian with variance proxy
$\sum_t \|\Delta\mu_t\|^2/\sigma_t^2 \le T\,D^2/\sigma_{\min}^2$.
This quantifies that our gate and gradient clipping keep both the \emph{magnitude} of mean shifts
and the \emph{variance} of the accumulated evidence controlled. See proofs on Appendix~\ref{app:llr-decomp} and ~\ref{app:subg}.


\subsection{Distributional Control and Return Bounds}
\label{sec:theory-ood-return}

Let $\mathcal{S}(s)$ denote the dataset action support at state $s$, and define the state-conditional OOD rate of a policy $\pi$ by
\[
\eta(\pi\,|\,s)=\Pr_{a\sim\pi(\cdot|s)}[a\notin\mathcal S(s)],\quad 
\eta(\pi)=\mathbb E_{s\sim d_{\text{eval}}}\![\eta(\pi\,|\,s)].
\]
Let $\hat Q$ be a learned critic and $Q^{\text{true}}$ the environment action-value.\footnote{Any fixed evaluation state distribution $d_{\text{eval}}$ may be used; in our experiments it is the dataset state marginal.} Assume the standard offline error split
\[\tiny
\varepsilon_{\rm in} := \sup_{a\in\mathcal S(s)}|\hat Q(s,a)-Q^{\text{true}}(s,a)|,\quad \varepsilon_{\rm out}:= \sup_{a\notin\mathcal S(s)}|\hat Q(s,a)-Q^{\text{true}}(s,a)|,
\]
and that $Q^{\text{true}}$ is $L$-Lipschitz in $a$. Intuitively, $\nu:=\varepsilon_{\rm out}- \varepsilon_{\rm in}\geq 0$ in offline RL due to extrapolation error. Additionally, write $\pi_{\rm LRT}$ for the LRT-gated policy and $\pi_Q$ for pure $Q$-guided sampling, with $a_{\rm LRT}\!\sim\!\pi_{\rm LRT}(\cdot|s)$ and $a_Q\!\sim\!\pi_Q(\cdot|s)$.

\begin{proposition}[Return comparison under offline errors]
\label{prop:gap}

Under the assumption above, let 
\begin{align*}
\Delta_{Q^{true}}:&=\mathbb E_s[Q^{true}(s,a_{LRT})-Q^{true}(s,a_Q)],\\ \Delta_{\hat Q}:&=\mathbb E_s[\hat Q(s,a_{LRT})-\hat Q(s,a_Q)],
\end{align*}
the following inequality holds: 
\[\boxed{\Delta_{Q^{true}}\geq \Delta_{\hat Q}-2\varepsilon_{\rm in} -\nu\,\left(\eta(\pi_Q)+\eta(\pi_{\rm LRT})\right).}\]

\end{proposition}

\begin{proof}[Sketch]
Decompose expectations into on-support and OOD parts. On-support deviations are bounded by $\varepsilon_{\rm in}$, off-support by $\varepsilon_{\rm out}$. Refer to Appendix ~\ref{appendix:bound} for complete proof.
\end{proof}

\emph{Assumption (Monotone support w.r.t.\ gating).} To relate $\eta(\pi_{\rm LRT})$ to the calibrated level $\alpha$, we use a mild monotonicity assumption. With the background head fixed, opening an LRT gate at any step does not decrease the probability that the final $a_0$ is outside $\mathcal S(s)$; if all $T$ gates reject, the final $a_0$ lies in $\mathcal S(s)$ with high probability.

\begin{proposition}[Trajectory-level OOD bound via single LLR threshold]
\label{prop:eta}
Assuming monotone support, let $\pi_{\rm LRT}$ be the LRT-gated policy
with cumulative LLR threshold $\tau$ calibrated under $H_0$ at level $\alpha$.
Then the expected state-conditional OOD rate satisfies
\[
\eta(\pi_{\rm LRT})
\;=\;
\mathbb E_{s}\big[\Pr_{a_0\sim\pi_{\rm LRT}(\cdot|s)}[a_0\notin \mathcal S(s)]\big]
\;\lesssim\;
\Pr_{H_0}(\ell_{\rm cum}\ge \tau) \;\approx\; \alpha.
\]
Here, $\ell_{\rm cum}$ is the cumulative log-likelihood ratio over the reverse trajectory,
and the approximation holds up to finite-sample calibration error (Thm.~\ref{thm:dkw}). Proof see Appendix~\ref{app:proof5.7}
\end{proposition}
Thus, smaller $\alpha$ directly yields fewer activations and lower expected OOD.

\subsection{A sufficient condition for LRT to dominate Q}

Combining Prop.~\ref{prop:gap} and Prop.~\ref{prop:eta} yields
\[\Delta_{Q^{\rm true}}
\ \ge\
\Delta_{\hat Q}
\ -\ 2\varepsilon_{\rm in}
\ -\ \nu\big(\alpha+\eta(\pi_Q)\big).\]
Hence, if
\[\alpha\ \le\ \alpha_{\max}
\ :=\ \frac{\Delta_{\hat Q}\ -\ 2\varepsilon_{\rm in}\ -\ \nu\,\eta(\pi_Q)}{\nu\,}\,,\]
then $\mathbb E_s[Q^{\rm true}(s,a_{\rm LRT})]\ge \mathbb E_s[Q^{\rm true}(s,a_Q)]$. Moreover, $\alpha_{\max}>0$ if and only if
$\Delta_{\hat Q}>2\varepsilon_{\rm in}+\nu\,\eta(\pi_Q)$. 

The bound is sufficient (and conservative): it uses the calibrated bound $\eta(\pi_{\rm LRT})\!\lesssim\!\alpha$ and the critic gap $\Delta_{\hat Q}$, which may be biased off--support. Nevertheless it yields actionable levers:
(i) decreasing $\beta_{\max}$ reduces $\nu$ and enlarges the feasible range; 
(ii) making $\eta(\pi_Q)$ small (a conservative anchor for the $Q$--step) increases $\alpha_{\max}$; 
(iii) evidence--tied guidance near $\mu_{\text{LRT}}$ can raise $\Delta_{\hat Q}$. 
In the favorable regime where $\eta(\pi_Q)$ is small, $\alpha_{\max}$ grows as
$\alpha_{\max}=(\Delta_{\hat Q}-2\varepsilon_{\rm in}-\nu\,\eta(\pi_Q))/\nu$, so moderate $\alpha$ can still be certified. 
In practice we treat $\widehat{\alpha}_{\max}$ as a diagnostic and sweep $\alpha$ on a log grid; the selected $\alpha$ is the knee of the return--risk curve (Fig.~\ref{pic:alba}), while the certificate provides a sanity upper bound. For further discussion, see Appendix~\ref{app:discussion-dominate}


\section{Experiments}
\label{sec:experiments}

We empirically study whether LRT guidance delivers a calibrated, \emph{interpretable} risk knob at inference time and yields a calibrated return–risk frontier relative to standard $Q$-guided sampling. Unless otherwise noted, diffusion training is vanilla (Sec.~\ref{sec:diffusion-training}); all risk control is applied \emph{only} at inference via the calibrated LRT gate.

\subsection{Tasks and Datasets}
We evaluate on continuous‑control D4RL MuJoCo tasks \cite{Fu2020D4RL}. Throughout, we standardize states and actions using dataset statistics. When we interact with the environment we map actions back to the original scale. For each task, we standardize states/actions using dataset means/stds and adopt the D4RL raw return (higher is better). 

We report three metrics. \emph{(i) Return:} per seed we evaluate $N_{\mathrm{roll}}$ episodes and compute the mean return, and tables report mean$\pm$std across seeds;
\emph{(ii) Realized Type‑I:} under $H_0$ with the deployed sampler, the frequency $\Pr[\ell_{\mathrm{cum}}\ge\hat\tau]$ (target $\alpha$), shown in Figure~\ref{pic:alba};
\emph{(iii) State-conditional OOD:} a $k$-NN proxy flagging actions as OOD if they exceed the $q$-th percentile distance of in-dataset $k$ state-neighbors (default: $k=50$, $q=95\%$).

Our goal is to isolate the effect of \emph{inference-time guidance} rather than to push absolute SOTA returns. 
Accordingly, we compare a standard $Q$-guided sampler (QG) against our LRT-guided sampler (LRT) and their simple composition (LRT+Q) under the same vanilla training pipeline, on representative D4RL MuJoCo tasks. 
This design controls for training confounders and highlights the \emph{structure} of the sampler: does replacing a heuristic push by a calibrated, level-$\alpha$ gate improve the return--OOD trade-off? 
We therefore report raw return alongside a state-conditional OOD metric and realized Type-I, and interpret results through risk--performance curves and Pareto fronts rather than absolute leaderboards.

\subsection{Baselines}
\begin{itemize}
\item \textbf{LRT:} evidence-gated sampler (no $Q$ step).
\item \textbf{QG:} standard action-space $Q$ update with schedules/clipping.
\item \textbf{LRT+Q:} LRT mean + small $Q$ step.
\end{itemize}

\subsection{Implementation \& Reproducibility}
\label{sec:exp-impl}
\textit{\textbf{Setup.}}
 IQL critic (2$\times$256 MLP, $\gamma{=}0.99$, expectile 0.7); advantages on standardized $(s,a)$; labels: global top-$p$ ($p{=}0.2$). Diffusion: $T{=}50$, DDPM linear noise $(1{-}\alpha_t)$ with endpoints $10^{-4}\!\to\!2{\times}10^{-2}$, MLP backbone (SiLU), two $\epsilon$ heads; AdamW $2{\times}10^{-4}$, batch 1024, 150 epochs. Class balancing via EMA positive rate; optional within-positive soft weights ($\tau_A$, $u_{\max}$). \emph{Full configs in Appx.~\ref{app:exp-impl}.}

\textit{\textbf{Inference \& calibration.}}
 Use reverse/posterior variance $\tilde\sigma_t^2$ for both heads and LLR; gate defaults $\beta_{\max}{=}1$, $\delta{\in}[1,2]$. Calibrate a single threshold $\hat\tau$ on a held-out state set by the fixed-point update
\[
\tau \leftarrow \mathrm{Quantile}_{1-\alpha}\{\ell_{\mathrm{cum}}(\tau)\}
\]
($K{=}6$ iterations with light momentum), then \emph{reuse} $\hat\tau$ for all deployments with the same $(\beta_{\max},\delta)$ and $Q$-composition.

\textit{\textbf{Protocol \& compute.}}
 Each configuration: $10$ seeds $\times$ $10$ episodes over the entire training and evaluation time; report mean $\pm$ std; actions mapped back to env scale and clipped; using D4RL-raw returns. Calibration uses $n \in[3000,5000]$ reverse trajectories.

\subsection{Main Results}
\label{sec:exp-main}
We perform structured hyperparameter sweeps to characterize the return–risk frontier. For \textbf{LRT}, we sweep the risk level $\alpha\in\{0.20,0.10,0.01,0.005,0.001\}$ with fixed $(\beta_{\max},\delta)$. For \textbf{QG}, we sweep the maximum guidance step $\lambda_{\max}\in\{0.2,0.1,0.05,0.02,0.005\}$ (same schedule and clipping). For \textbf{LRT+Q}, we sweep the Cartesian product $\alpha\times\lambda_{\max}$ (25 settings per seed). 

Table~\ref{tab:main} reports a single operating point per method by selecting the setting that maximizes the \emph{mean return across seeds} among the swept candidates, and then reporting mean$\pm$std across seeds; Fig.~\ref{pic:alba} shows the full sweeps as risk--performance curves and Pareto fronts. To mitigate seed-to-seed variability, we also report paired relative gaps vs.\ Q, see Table~\ref{tab:relgap} in Appendix. Two notable cases are: (i) \texttt{halfcheetah-medium-replay-v2} exhibits substantially higher seed-to-seed variability, especially for value-guided samplers; and (ii) \texttt{walker2d-medium-v2} is an exception where LRT achieves lower return and higher OOD than value-guided baselines under the selected hyperparameters. We discuss both phenomena and report relative-gap statistics in Appx.~\ref{app:exp-disc} (see also Appx.~\ref{app:exp-variability}, \ref{app:walker2d-exception}).

\begin{table}[t]
\caption{\textbf{Main results on D4RL MuJoCo.} For each task/dataset and method, we report return (D4RL raw return) and state-conditional OOD (reported in units of $\times 10^{-2}$; lower is better). Results correspond to the operating point that maximizes mean return across seeds within each method’s sweep.}
\label{tab:main}
\centering
\small
\begin{tabular}{lccc}
    \toprule
    Data Name & Mode & Return$\pm$std & OOD$(10^{-2})\pm$std \\
    \midrule
    hopper-medium-replay-v2 & lrt   & 329$\pm$22 & \textbf{1.84$\pm$0.15} \\
    hopper-medium-replay-v2 & q     & 363$\pm$34 & 6.32$\pm$0.53 \\
    hopper-medium-replay-v2 & lrt+q & \textbf{366$\pm$28} & 6.67$\pm$0.77 \\
    \midrule
    halfcheetah-medium-replay-v2 & lrt   & 558$\pm$111 & \textbf{1.14$\pm$0.25} \\
    halfcheetah-medium-replay-v2 & q     & 598$\pm$216 & 13.13$\pm$2.72 \\
    halfcheetah-medium-replay-v2 & lrt+q & \textbf{615$\pm$252} & 13.93$\pm$3.20 \\
    \midrule
    walker2d-medium-replay-v2 & lrt   & 315$\pm$21 & \textbf{0.41$\pm$0.07} \\
    walker2d-medium-replay-v2 & q     & 373$\pm$44 & 3.76$\pm$0.43 \\
    walker2d-medium-replay-v2 & lrt+q & \textbf{375$\pm$46} & 3.86$\pm$0.87 \\
    \specialrule{0.1em}{0.4em}{0.4em}
    hopper-medium-v2 & lrt   & 744$\pm$71  & \textbf{9.21$\pm$0.46} \\
    hopper-medium-v2 & q     & 1176$\pm$60 & 9.30$\pm$0.55 \\
    hopper-medium-v2 & lrt+q & \textbf{1197$\pm$47} & 11.32$\pm$0.83 \\
    \midrule
    halfcheetah-medium-v2 & lrt   & 3526$\pm$38 & \textbf{3.61$\pm$0.97} \\
    halfcheetah-medium-v2 & q     & 4404$\pm$31 & 5.11$\pm$0.80 \\
    halfcheetah-medium-v2 & lrt+q & \textbf{4452$\pm$54} & 5.02$\pm$0.87 \\
    \midrule
    walker2d-medium-v2 & lrt   & 568$\pm$49 & 10.79$\pm$2.51 \\
    walker2d-medium-v2 & q     & 2282$\pm$178 & 5.28$\pm$0.96 \\
    walker2d-medium-v2 & lrt+q & \textbf{2448$\pm$196} & \textbf{4.94$\pm$0.95} \\
    \bottomrule
\end{tabular}
\end{table}

\begin{figure}[t]
  \centering
  \text{\small hopper-medium-replay-v2}\\
  \includegraphics[width=.42\columnwidth,trim=10 10 10 10,clip]{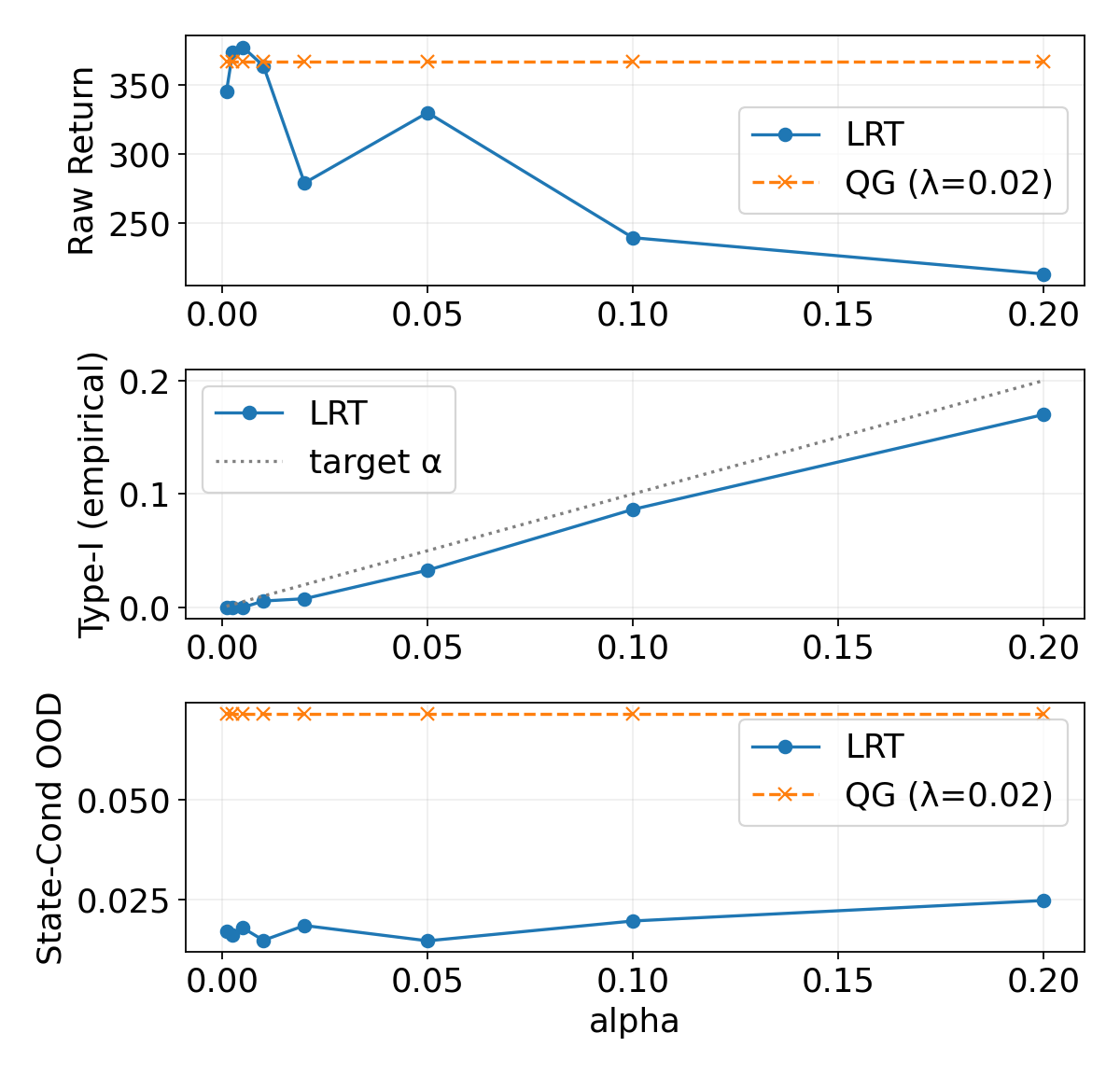}\hfill
  \includegraphics[width=.54\columnwidth,trim=10 10 10 13,clip]{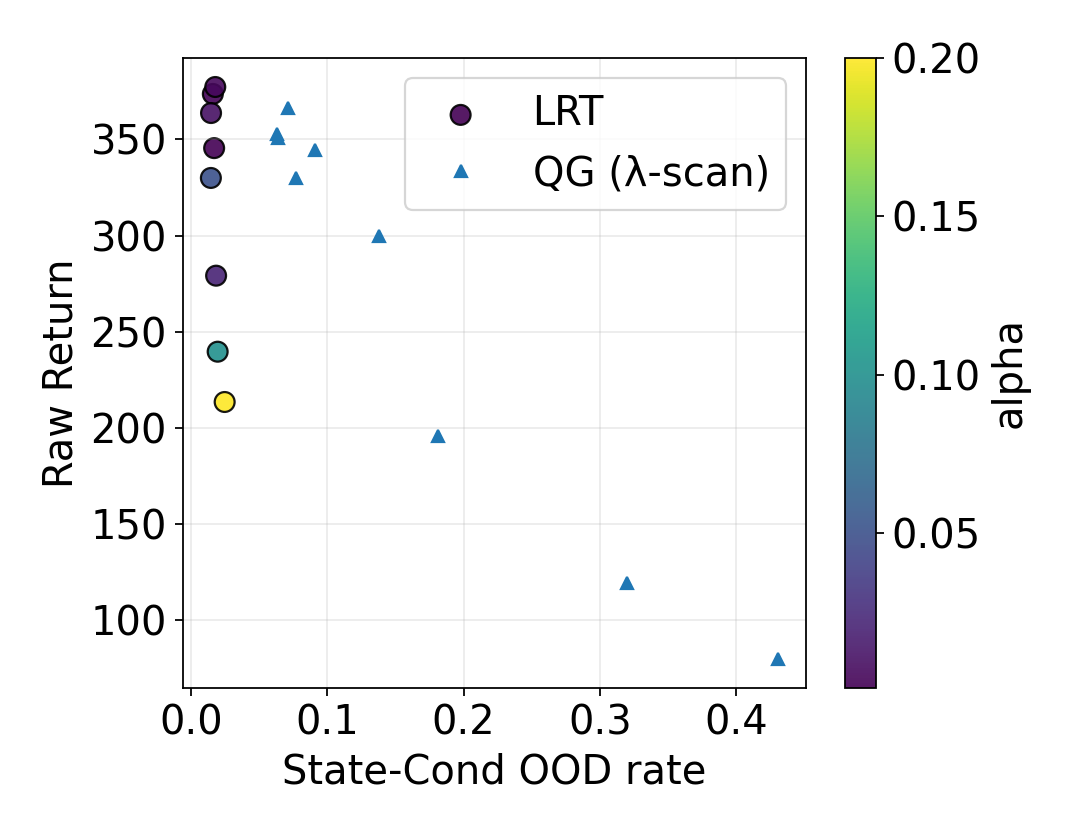}\\
  \text{\small halfcheetah-medium-replay-v2}\\
    \includegraphics[width=.40\columnwidth,trim=10 10 10 10,clip]{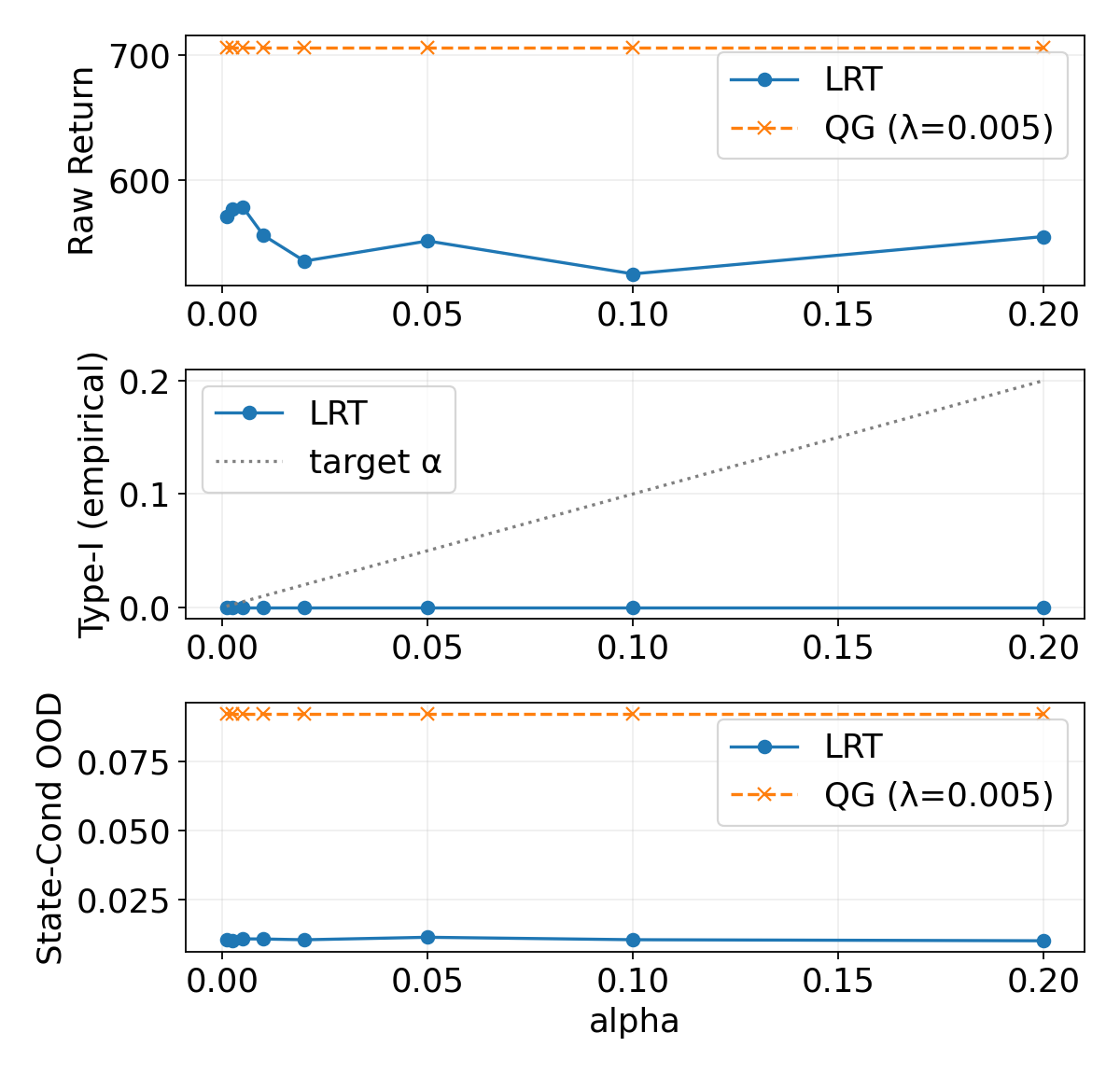}\hfill
  \includegraphics[width=.52\columnwidth,trim=6 6 6 6,clip]{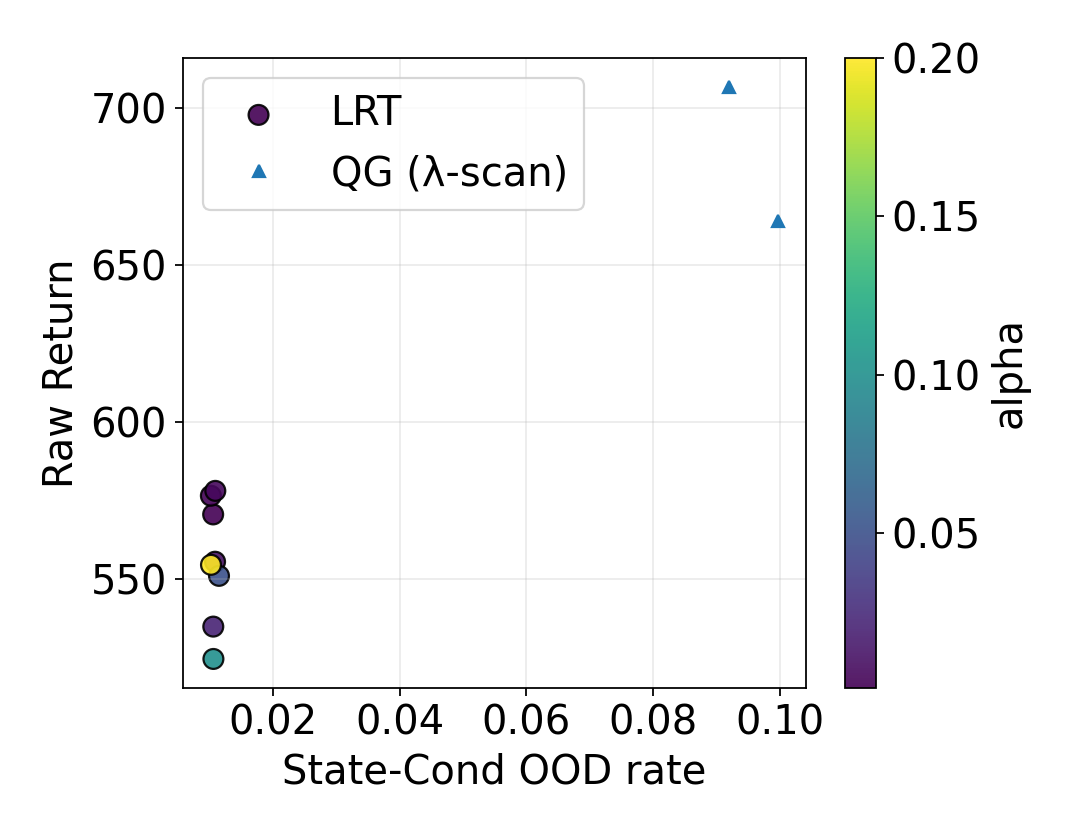}\\
  \text{\small walker2d-medium-replay-v2}\\
    \includegraphics[width=.40\columnwidth,trim=10 10 10 10,clip]{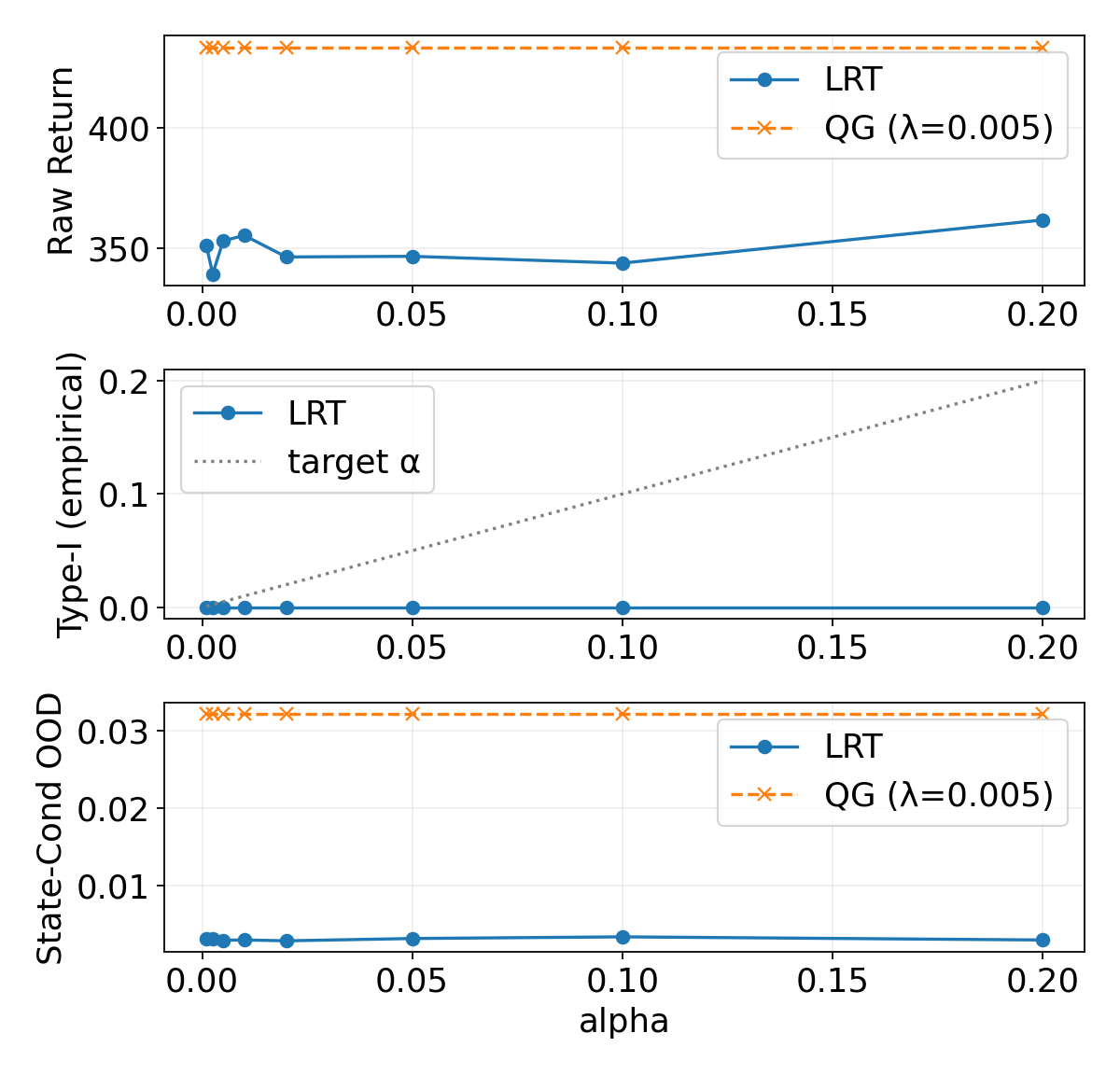}\hfill
  \includegraphics[width=.52\columnwidth,trim=6 6 6 6,clip]{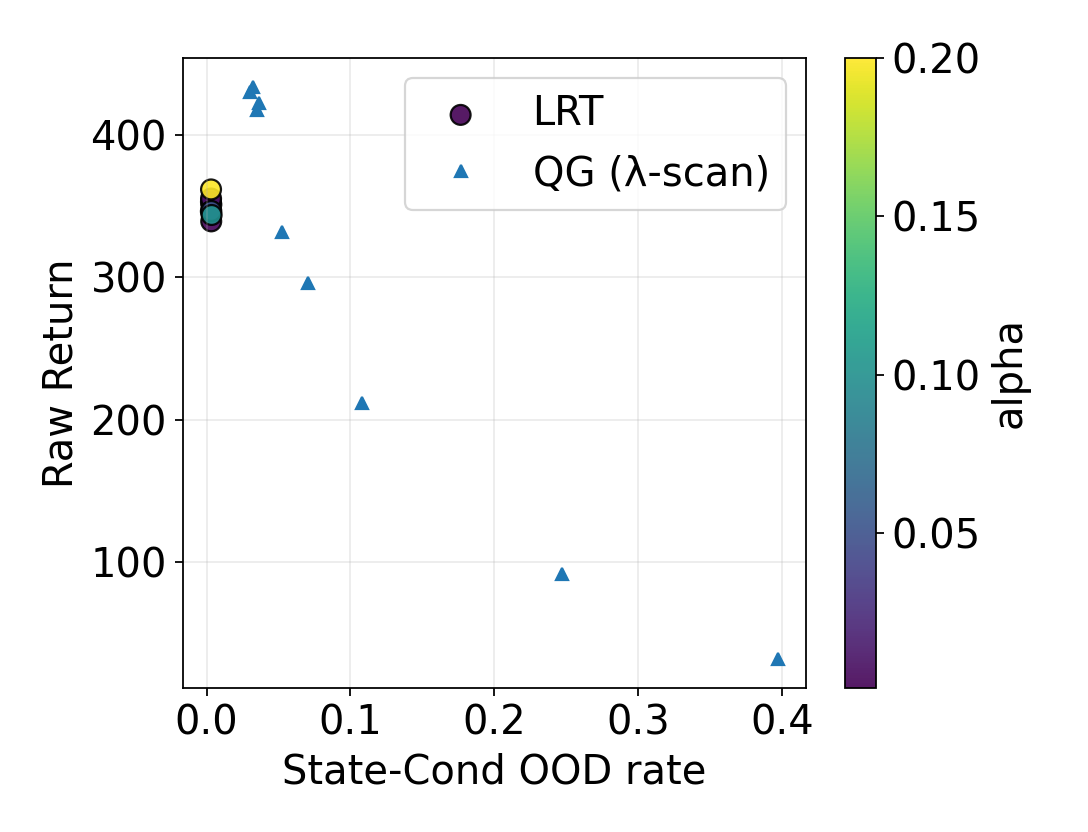}
  \caption{\small\textbf{Risk–performance and Pareto fronts across tasks.}
Left of each row: risk–performance curves versus target $\alpha$.
We report return (top), realized Type-I (middle), and state-conditional OOD (bottom) for LRT (solid) and QG (dashed);
the realized Type-I tracks the target (gray) within finite-sample DKW bands.Right of each row: Pareto fronts (OOD vs.\ return; color encodes $\alpha$). LRT shifts the frontier up-and-left relative to QG on tasks where off-support critic error dominates, yielding higher return at lower OOD for the same $\alpha$. Error bars denote standard errors over evaluation rollouts.}
\label{pic:alba}
\end{figure}

\textit{\textbf{Results at a glance.}}
Across tasks, \textbf{LRT} often yields the lowest state-conditional OOD, acting as a conservative, low-risk anchor. Adding a small critic step (\textbf{LRT+Q}) typically increases return but can increase OOD, tracing a return–risk frontier. Notably, on \texttt{medium-replay} datasets LRT’s OOD advantage is most pronounced, while on \texttt{medium} datasets value guidance often delivers larger return gains. See Appx.~\ref{app:exp-disc} for per-task interpretation.

\textit{\textbf{Reading the curves and fronts.}}
In the left panels of Fig.~\ref{pic:alba}, the realized Type-I (middle row) closely tracks the target
$\alpha$ within DKW bands, validating calibration. As $\alpha$ decreases, return (top) drops modestly
while state-conditional OOD (bottom) decreases consistently, exhibiting a smooth, monotone risk knob.
The right panels (Pareto fronts) make the trade-off explicit: on Hopper and Walker2d, LRT shifts the
frontier up-and-left relative to QG—higher return at lower OOD for a fixed $\alpha$—whereas on
HalfCheetah, adding a small $Q$-step can push return further at the cost of OOD, matching the intended
“anchor-plus-exploitation” behavior.

\subsection{Ablations}\label{sec:Ablation}
\underline{Effect of $\alpha$:}
Smaller $\alpha$ $\Rightarrow$ larger $\hat\tau$, fewer gate activations, \emph{lower} state-conditional OOD with \emph{modest} return drop (consistent with the left panels in Fig.~\ref{pic:alba}). \underline{Cap $\beta_{\max}$ \& temperature $\delta$:}
 Decreasing $\beta_{\max}$ contracts updates toward $\mu_u$ (uniform OOD reduction even at large $\alpha$); decreasing $\delta$ sharpens the switch and approaches hard LRT (shifting Pareto up–left). \underline{Labeling $p$:}
 Moderate $p\in[0.1,0.3]$ is robust; calibration preserves Type-I semantics regardless of $p$, so Pareto trends remain unchanged (see Appendix~\ref{app:p-ablation}). \underline{Evidence dynamics:}
Prefix LLRs typically hover near zero at early high-noise steps and exceed $\hat\tau$ only later, if at all; the gate $\beta_t/\beta_{\max}$ remains near zero before crossing and then rises smoothly (Fig.~\ref{fig:traces}). This matches the monotone risk knob seen in Fig.~\ref{pic:alba}. \underline{Stress test under variance mismatch:}
We conduct a stress test under deliberate variance mis-specification of the background head at inference. Increasing mismatch inflates Type-I error and OOD, causing return degradation; matched variance ($s=1$) preserves risk control, while severe mismatch ($s\ge 2$) degrades guarantees. See Appendix~\ref{app:var-stress} for details.

\begin{figure}[h]
\centering
  \includegraphics[width=7cm,height=3cm,trim=0 25 0 10,clip]{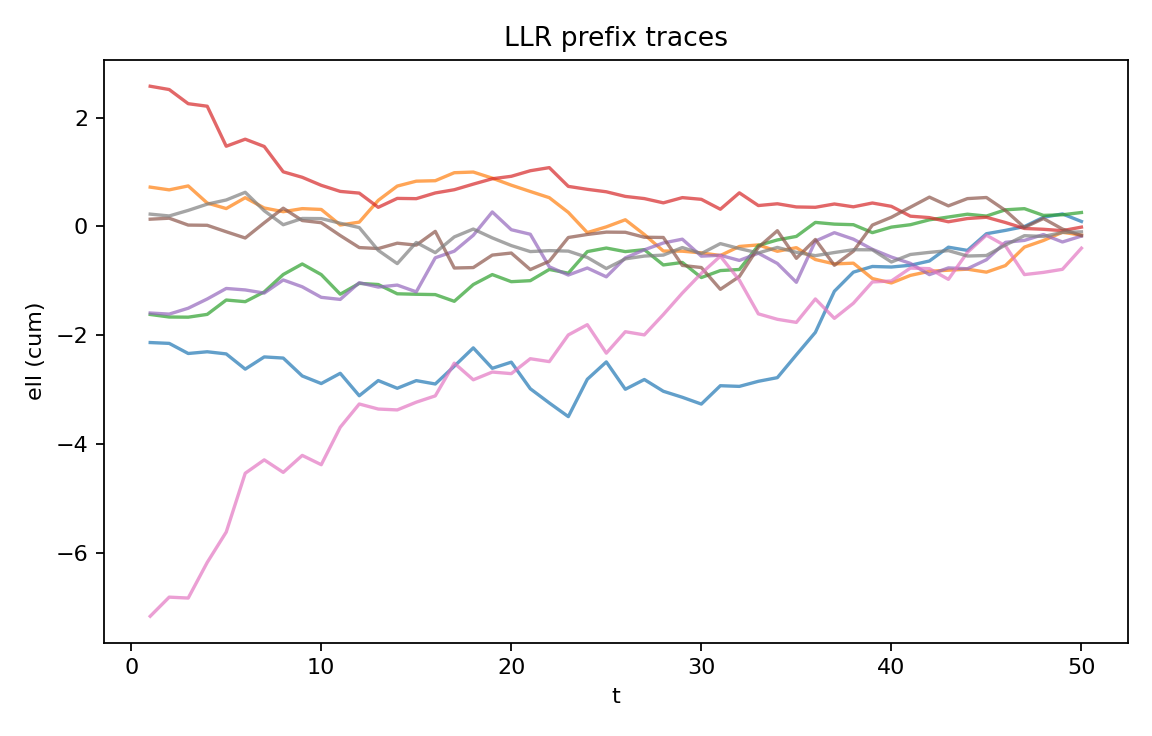}%
\caption{\textbf{LLR traces.} Prefix LLR across denoising steps on random states; lines are different trajectories. Data: \texttt{hopper-medium-replay-v2}.}
\label{fig:traces}
\end{figure}

\section{Conclusion}
\label{sec:conclusion}

We presented \textbf{LRT-Diffusion}, a calibrated, inference-only guidance rule that turns diffusion-policy denoising into an evidence-gated process governed by a single, interpretable risk knob~$\alpha$.  
Under equal covariances, the hard likelihood-ratio test (LRT) is uniformly most powerful at level~$\alpha$; our smooth gate retains the same statistical semantics through a simple Monte-Carlo calibration matched to the deployed sampler.  
Empirically, LRT delivers calibrated control of Type-I risk and often improves the return–OOD trade-off over standard $Q$-guided sampling; moreover, in the experiments in Table 1, the composed sampler (LRT+Q) consistently achieves the highest return, highlighting that evidence-gated guidance can stabilize and amplify the benefits of value-gradient updates.
Together, these demonstrate that \emph{risk-aware diffusion guidance} can be achieved entirely at inference without modifying training objectives.



\balance


\bibliographystyle{ACM-Reference-Format} 
\bibliography{reference}


\newpage

\appendix

\section{Detailed pipeline of LRT-Diffusion}\label{app:detail}
For a detailed illustration of the inference-time pipeline, see Fig.~\ref{fig:app-detail} in the Appendix.

\begin{figure*}[t]
  \centering
  \includegraphics[width=\textwidth]{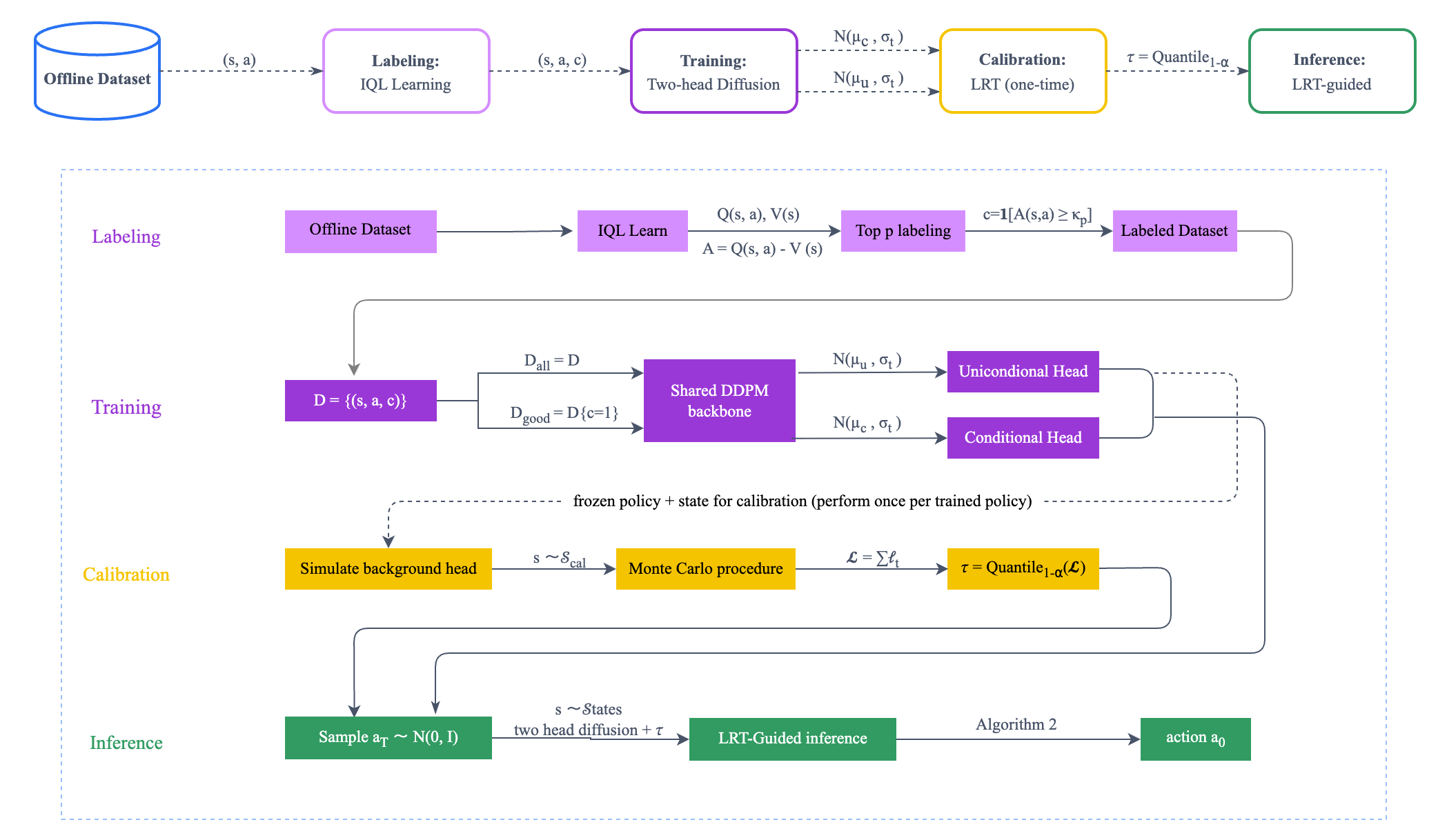}
  \caption{
  Overall pipeline of LRT-Diffusion (reproduced for reference).
  Training is vanilla; all risk control is applied at inference via a calibrated likelihood-ratio gate.
  }
  \label{fig:app-detail}
\end{figure*}

\section{Gate Selection}
\label{app:gate}

\subsection{Hard Gate-UMP: Proof of Proposition \ref{prop:np-ump}}

\begin{proof}

Define the (trajectory) likelihood ratio
\[
L(a_{0:T}) \;:=\; \frac{p_c(a_{0:T}\mid s)}{p_u(a_{0:T}\mid s)}
\;=\; \exp\!\big(\ell_{\mathrm{cum}}(a_{0:T})\big).
\]
Let $\Phi$ be the class of (possibly randomized) tests
$\phi:\mathcal A^{T+1}\!\to[0,1]$ with size at most $\alpha$,
i.e.\ $\mathbb E_{H_0}[\phi]\le \alpha$.
By the Neyman--Pearson lemma for simple $H_0$ vs.\ $H_1$, there exists
$k>0$ such that the likelihood-ratio test
$\phi^\star(a_{0:T})=\mathbf 1\{L(a_{0:T})\ge k\}$ (with possible
randomization on the boundary $\{L=k\}$) attains size exactly $\alpha$
and maximizes the power $\mathbb E_{H_1}[\phi]$ over $\Phi$.
Since $\log(\cdot)$ is monotone, thresholding $L$ is equivalent to
thresholding $\ell_{\mathrm{cum}}=\log L$. Hence the stated test is
most powerful at level $\alpha$.
\end{proof}
\emph{Remark.} NP optimality holds for simple $H_0$ vs.\ $H_1$ regardless of covariance structure; the equal-covariance assumption is used only to compute $\ell_t$ in closed form.
\subsection{Soft to Hard Limit: Proof of Lemma ~\ref{lem:soft-hard}}\label{app:proof-soft-hard}

\begin{proof}
Recall from Sec.~\ref{sec:gates} the logistic gate
\[
\beta_t(\ell_{\mathrm{cum}};\tau,\delta,\beta_{\max})
~=~ \beta_{\max}\,\sigmoid{\tfrac{\ell_{\mathrm{cum}}-\tau}{\delta}},
\qquad
\sigmoid{x}=\tfrac{1}{1+e^{-x}},
\]
and the proposal mean at step $t$,
\(
\mu_t^{\mathrm{soft}}=\mu_u+\beta_t(\mu_c-\mu_u)
\),
with reverse variance $\sigma_t^2 I$.
We first establish the pointwise limit of $\beta_t$ and hence of $\mu_t^{\mathrm{soft}}$,
then show weak convergence of the trajectory law by a coupling argument.

\medskip\noindent\textbf{Step 1: Pointwise limit of the gate and the mean.}
Fix $(\tau,\beta_{\max})$.
For any $x\in\mathbb{R}$ with $x\neq \tau$,
\[
\lim_{\delta\downarrow 0}\ \sigmoid{\tfrac{x-\tau}{\delta}}
~=~ \mathbf{1}\{x\ge \tau\}.
\]
Hence, for any realization of the running evidence $\ell_{\mathrm{cum}}$ at step $t$ such that
$\ell_{\mathrm{cum}}\neq \tau$,
\[
\beta_t(\ell_{\mathrm{cum}};\tau,\delta,\beta_{\max})
~\xrightarrow[\delta\downarrow 0]{}~
\beta_{\max}\,\mathbf{1}\{\ell_{\mathrm{cum}}\ge \tau\}.
\]
Therefore
\[
\mu_t^{\mathrm{soft}}
~=~\mu_u+\beta_t(\mu_c-\mu_u)
~\xrightarrow[\delta\downarrow 0]{}~
\mu_u+\beta_{\max}\,\mathbf{1}\{\ell_{\mathrm{cum}}\ge \tau\}\,(\mu_c-\mu_u)
~=:\ \mu_t^{\mathrm{hard}}.
\]
It remains to note that, under our reverse kernels with Gaussian noise,
the scalar $\ell_{\mathrm{cum}}$ is a continuous function of the Gaussian seeds,
so $\mathbb{P}(\ell_{\mathrm{cum}}=\tau)=0$; hence the convergence above holds almost surely.

\medskip\noindent\textbf{Step 2: Weak convergence of the trajectory law.}
Couple all samplers (for each $\delta\ge 0$) by using the \emph{same} initial latent
$a_T\sim q(\cdot)$ and the \emph{same} Gaussian noise draws $z_T,\ldots,z_1$:
\begin{align*}
a_{t-1}^{\mathrm{soft}} \;=\;& \mu_t^{\mathrm{soft}}\big(t,s,a_t^{\mathrm{soft}}\big)\;+\;\sigma_t z_t,
\\
a_{t-1}^{\mathrm{hard}} \;=\; &\mu_t^{{\mathrm{hard}}}\big(t,s,a_t^{\mathrm{hard}}\big)\;+\;\sigma_t z_t.
\end{align*}
We prove by backward induction on $t=T,\ldots,1$ that
$a_t^{\mathrm{soft}}\to a_t^{\mathrm{hard}}$ and the running evidence prefixes
$\ell_{\mathrm{cum}}^{\mathrm{soft}}\to \ell_{\mathrm{cum}}^{\mathrm{hard}}$ almost surely as $\delta\downarrow 0$.

\emph{Base case $t=T$:} By construction $a_T^{\mathrm{soft}}=a_T^{\mathrm{hard}}$ almost surely.

\emph{Inductive step:}
Assume $a_t^{\mathrm{soft}}\to a_t^{\mathrm{hard}}$ and
$\ell_{\mathrm{cum}}^{\mathrm{soft}}\to \ell_{\mathrm{cum}}^{\mathrm{hard}}$ almost surely.
The maps $(t,s,a)\mapsto \mu_u(t,s,a)$ and $\mu_c(t,s,a)$ are continuous (shared backbone + MLP heads with standard activations), hence
\[
\mu_u\big(t,s,a_t^{\mathrm{soft}}\big)\to \mu_u\big(t,s,a_t^{\mathrm{hard}}\big),\quad
\mu_c\big(t,s,a_t^{\mathrm{soft}}\big)\to \mu_c\big(t,s,a_t^{\mathrm{hard}}\big)
\quad\text{a.s.}
\]

By Step~1 and the a.s.\ convergence of $\ell_{\mathrm{cum}}^{\mathrm{soft}}$,
$$\beta_t^{\mathrm{soft}}=\beta_t(\ell_{\mathrm{cum}}^{\mathrm{soft}};\tau,\delta,\beta_{\max})
\to \beta_{\max}\mathbf{1}\{\ell_{\mathrm{cum}}^{\mathrm{hard}}\ge\tau\}$$
almost surely (the boundary has probability zero). Therefore
\[
\mu_t^{\mathrm{soft},(\delta)}
~\longrightarrow~
\mu_t^{\mathrm{hard}}
\quad\text{a.s.}
\]
With the coupled noise $z_t$, we obtain
\[
a_{t-1}^{\mathrm{soft}}=\mu_t^{\mathrm{soft}}+\sigma_t z_t
~\longrightarrow~
\mu_t^{\mathrm{hard}}+\sigma_t z_t
= a_{t-1}^{\mathrm{hard}}
\quad\text{a.s.}
\]
Finally, the one-step LLR increment $\ell_t$ is a continuous function of
$(a_t,a_{t-1})$ under equal covariances, hence
$\ell_t^{\mathrm{soft}}\to \ell_t^{\mathrm{hard}}$ and the updated evidence
$\ell_{\mathrm{cum}}^{\mathrm{soft}}\to \ell_{\mathrm{cum}}^{\mathrm{hard}}$ almost surely,
closing the induction.

\medskip
Thus $a_{0:T}^{\mathrm{soft}}\to a_{0:T}^{\mathrm{hard}}$ almost surely under the coupling.
Almost-sure convergence under a coupling implies weak convergence of the induced
trajectory laws, so the sampler with soft gate converges (in distribution) to the
sampler with the hard gate as $\delta\downarrow 0$.
\end{proof}

\section{Proofs for Finite-sample Calibration and Stability}
\label{app:calib-stability}

\paragraph{Frozen sampler and notation.}
Fix the deployed (frozen) sampler, including $(\beta_{\max},\delta)$ and, if enabled, the deterministic $Q$-composition (schedule, clipping, and the evaluation center). Under $H_0$ the trajectory $a_{T:0}$ is generated by this sampler using the unconditional head. Let $F_0$ be the CDF of the cumulative LLR $\ell_{\mathrm{cum}}$ under $H_0$, and let $\widehat F_n$ be the empirical CDF from $n$ i.i.d.\ calibration runs. The plug-in threshold is $\hat\tau=\inf\{x:\widehat F_n(x)\ge 1-\alpha\}$.

\subsection{Finite-sample Calibration: Proof of Theorem~\ref{thm:dkw} }\label{app:proof-dkw}
\begin{proof}
Let $F_0$ be the CDF of the cumulative LLR $\ell_{\mathrm{cum}}$ under $H_0$ for the \emph{frozen} sampler
(soft gate and, if present, the fixed $Q$-step). Given i.i.d.\ calibration draws
$\ell^{(1)},\dots,\ell^{(n)}$ and the empirical CDF $\widehat F_n$, define the plug-in quantile
$\hat\tau=\inf\{x:\widehat F_n(x)\ge 1-\alpha\}$.

According to Dvoretzky–Kiefer–Wolfowitz 
\begin{equation}
\label{eq:dkw-leq}
\mathbb Pr\left(\sup_x|F_0(x)-\widehat F_n(x)|\geq \varepsilon_n \right) \leq 2 e^{-2n\epsilon_n^2}
\end{equation}
Let $\zeta = 2 e^{-2n\epsilon_n^2}$ denote the lower bound for eq~(\ref{eq:dkw-leq}). Then,
\begin{equation}
\mathbb Pr\left(\sup_x|F_0(x)-\widehat F_n(x)|\leq \varepsilon_n \right) \geq 1-2 e^{-2n\epsilon_n^2}=1-\zeta
\end{equation}

Hence, $F_0(\hat\tau)\ge \widehat F_n(\hat\tau)-\varepsilon_n\ge 1-\alpha-\varepsilon_n$ and
$\mathbb{P}_{H_0}(\ell_{\mathrm{cum}}\ge \hat\tau)\le \alpha+\varepsilon_n$. Moreover, when $1-2 e^{-2n\epsilon_n^2}\geq 1-\zeta$ ensure at least $1-\zeta$ that the prediction is accurate, which the sample size will follow $n \;\ge\; \frac{1}{2\varepsilon^2}\,\log\!\frac{2}{\zeta}$.
\end{proof}

\subsection{Deterministic Displacement bound: Proof for Lemma~\ref{lemma:deterministic-bound} }
\label{app:det-shift}
\begin{proof}
By definition of the (deterministic) update at step $t$,
\[
a_t-\mu_{u,t}
= \beta_t\,\Delta\mu_t \;+\; \lambda_t\,\sigma_t^2\,g_t.
\]
Taking Euclidean norms and applying the triangle inequality,
\[
\|a_t-\mu_{u,t}\|
\;\le\; |\beta_t|\,\|\Delta\mu_t\| \;+\; |\lambda_t|\,\sigma_t^2\,\|g_t\|.
\]
Using the bounds $0\le\beta_t\le\beta_{\max}$, $0\le\lambda_t\le\lambda_{\max}$ and
$\|g_t\|\le G$, we obtain
\[
\|a_t-\mu_{u,t}\|
\;\le\; \beta_{\max}\,\|\Delta\mu_t\| \;+\; \lambda_{\max}\,\sigma_t^2\,G
\;=:\; B_t,
\]
which is \eqref{eq:step-bound}.

If, in addition, $\|\Delta\mu_t\|\le D$ and $\sigma_t^2\le S^2$, then
\[
\|a_t-\mu_{u,t}\|
\;\le\; \beta_{\max} D \;+\; \lambda_{\max} S^2 G
\;=:\; B_{\text{step}}
\quad\text{for all }t.
\]
Finally, summing the per-step bounds yields the stated cumulative bound:
\[
\sum_{t=1}^T \|a_t-\mu_{u,t}\|
\;\le\; \sum_{t=1}^T B_t
\;\le\; T\,B_{\text{step}}.
\]
\end{proof}

\subsection{One-step LLR Decomposition and Moment Bounds}
\label{app:llr-decomp}
Under equal covariances $\sigma_t^2 I$, the one-step LLR admits the exact identity
\begin{equation}
\label{eq:llr-step-identity}
\Delta \ell_t
= -\frac{\|a_{t-1}-\mu_{c,t}\|^2-\|a_{t-1}-\mu_{u,t}\|^2}{2\sigma_t^2}
= \frac{\Delta\mu_t^\top\!\left(a_{t-1}-\tfrac{\mu_{c,t}+\mu_{u,t}}{2}\right)}{\sigma_t^2}.
\end{equation}
Write the stochastic update as $a_{t-1}=m_t+\sigma_t z_t$, where $z_t\sim\mathcal N(0,I)$ is independent of $z_{t'}$ for $t'\neq t$ and independent of the past given the sampler’s history (the filtration $\mathcal F_{t}$ generated by $(a_T,\ldots,a_t)$ and the gating/gradient choices). Substituting into \eqref{eq:llr-step-identity} yields the \emph{deterministic + Gaussian} decomposition
\[
\Delta\ell_t
=\underbrace{\frac{\Delta\mu_t^\top\!\left(m_t-\tfrac{\mu_{c,t}+\mu_{u,t}}{2}\right)}{\sigma_t^2}}_{=:~A_t \text{ (deterministic given }\mathcal F_t)}
\;+\;
\underbrace{\frac{\Delta\mu_t^\top z_t}{\sigma_t}}_{=:~B_t \text{ (zero-mean Gaussian)}}.
\]
\textbf{(i) Conditional mean bound.)} Since
$m_t-\tfrac{\mu_{c,t}+\mu_{u,t}}{2} = (m_t-\mu_{u,t})-\tfrac{1}{2}\Delta\mu_t$,
by Cauchy--Schwarz,
\[
|A_t| \;=\; \frac{1}{\sigma_t^2}\,\big|\langle \Delta\mu_t,\,(m_t-\mu_{u,t})-\tfrac{1}{2}\Delta\mu_t\rangle\big|
\ \le\ \frac{\|\Delta\mu_t\|}{\sigma_t^2}\Big(\|m_t-\mu_{u,t}\|+\tfrac{1}{2}\|\Delta\mu_t\|\Big).
\]
Taking expectation conditional on $\mathcal F_t$ gives
\[
\big|\mathbb E[\Delta\ell_t\mid \mathcal F_t]\big|
\ \le\ \frac{\|\Delta\mu_t\|}{\sigma_t^2}\Big(\|m_t-\mu_{u,t}\|+\tfrac{1}{2}\|\Delta\mu_t\|\Big)
\ \le\ \frac{\|\Delta\mu_t\|}{\sigma_t^2}\Big(B_t+\tfrac{1}{2}\|\Delta\mu_t\|\Big),
\]
where in the last inequality we used the displacement bound $\|m_t-\mu_{u,t}\|\le B_t$ from \S\ref{app:det-shift}.

\textbf{(ii) Conditional variance.)} Because $B_t=\frac{1}{\sigma_t}\Delta\mu_t^\top z_t$ with $z_t\sim\mathcal N(0,I)$,
\[
\mathrm{Var}(\Delta\ell_t\mid \mathcal F_t)=\mathrm{Var}(B_t\mid \mathcal F_t)
=\frac{\|\Delta\mu_t\|^2}{\sigma_t^2}.
\]
In particular, if $\|\Delta\mu_t\|\le D$ and $\sigma_t\ge \sigma_{\min}>0$, then
$\mathrm{Var}(\Delta\ell_t\mid \mathcal F_t)\le D^2/\sigma_{\min}^2$
and
\[
\big|\mathbb E[\Delta\ell_t\mid \mathcal F_t]\big|\ \le\ \frac{D}{\sigma_{\min}^2}\Big(B_{\text{step}}+\tfrac{D}{2}\Big)
\qquad\text{(using } \|m_t-\mu_{u,t}\|\le B_{\text{step}}\text{).}
\]
\qed

\subsection{Sub-Gaussian concentration for the cumulative LLR}
\label{app:subg}
Define the variance proxy
\[
V \;:=\; \sum_{t=1}^T \frac{\|\Delta\mu_t\|^2}{\sigma_t^2}.
\]
Let $\mathcal F_t$ be the sampler’s filtration up to step $t$. As above, given $\mathcal F_t$, we can write
$\Delta\ell_t=A_t+B_t$ with $B_t=\frac{1}{\sigma_t}\Delta\mu_t^\top z_t$, where $z_t\sim\mathcal N(0,I)$ is independent of $\mathcal F_t$ and of $z_{t'}$ for $t'\neq t$. Hence, conditionally,
\[
\mathbb E\big[e^{\lambda B_t}\mid \mathcal F_t\big]
=\exp\!\left(\frac{\lambda^2}{2}\cdot \frac{\|\Delta\mu_t\|^2}{\sigma_t^2}\right).
\]
Iterating conditional expectations (tower property) and using independence of the $z_t$ across $t$,
\[
\mathbb E\!\left[\exp\!\left(\lambda\sum_{t=1}^T(B_t-\mathbb E[B_t\mid \mathcal F_t])\right)\right]
\le \exp\!\left(\frac{\lambda^2}{2}\sum_{t=1}^T \frac{\|\Delta\mu_t\|^2}{\sigma_t^2}\right)
= \exp\!\left(\frac{\lambda^2}{2}V\right).
\]
Therefore $\sum_{t=1}^T \big(B_t-\mathbb E[B_t\mid \mathcal F_t]\big)$ is \emph{sub-Gaussian} with proxy $V$.
Adding the deterministic (given history) terms $A_t-\mathbb E[A_t\mid \mathcal F_t]$ only shifts the mean and does not affect the sub-Gaussian proxy. Consequently,
\[
\ell_{\mathrm{cum}}-\mathbb E[\ell_{\mathrm{cum}}]
=\sum_{t=1}^T \big(\Delta\ell_t-\mathbb E[\Delta\ell_t\mid \mathcal F_t]\big),
\]is sub-Gaussian with variance proxy V,
and for any $x>0$,
\[
\Pr\!\left[\ell_{\mathrm{cum}}-\mathbb E\ell_{\mathrm{cum}}\ge x\right]
\ \le\ \exp\!\left(-\frac{x^2}{2V}\right).
\]
Under the scale controls in \S\ref{app:det-shift}, $$V\le \sum_{t=1}^T D^2/\sigma_{\min}^2
= T\,D^2/\sigma_{\min}^2$$.
\qed

\paragraph{Summary of implications.}
The cap $\beta_{\max}$ and gradient clipping (via $G,\lambda_{\max}$) yield a per-step deterministic displacement $B_t$ and a uniform bound $B_{\text{step}}$; together with the equal-covariance identity, these give (i) explicit bounds on the conditional mean and variance of each LLR increment, and (ii) a sub-Gaussian concentration for the cumulative evidence with proxy $V=\sum_t \|\Delta\mu_t\|^2/\sigma_t^2$. These stability controls make the calibration reproducible (Theorem~\ref{thm:dkw}) and the gate numerically well-conditioned.

\section{When LRT Dominate Q guidance}
\subsection{Proof for Proposition ~\ref{prop:gap}}
\label{appendix:bound}

\begin{proof}

Write $\Lambda(s,a):=\hat Q(s,a)-Q^{\rm *}(s,a)$. Then, under the assumption in Sec~\ref{sec:theory-ood-return}, under policy $\pi$ we have
$$|\Lambda(s,a_\pi)|\leq\epsilon_{in}(1-\eta_\pi)+\epsilon_{out}\eta_\pi.$$
Lower bound for $\mathbb E_s[Q^{\rm *}(s,a_{\rm LRT})]$:
\begin{align}
\tiny
\mathbb E_s[Q^{\rm *}(s,a_{\rm LRT})] = & \mathbb E_s[\hat Q(s,a_{LRT})-\hat Q(s,a_{LRT})+Q^{\rm *}(s,a_{\rm LRT})]\\
\geq & E_s[\hat Q(s,a_{LRT})]-\Lambda(s,a_{LRT})\\
\geq & E_s[\hat Q(s,a_{LRT})]-|\Lambda(s,a_{LRT})|\\
\geq & E_s[\hat Q(s,a_{LRT})]-\epsilon_{in}-(\epsilon_{out}-\epsilon_{in})\eta_{LRT}.
\end{align}
Similarly, the upper bound for $\mathbb E_s[Q^{\rm *}(s,a_{\rm Q})]$:
\begin{align}
\tiny
\mathbb E_s[Q^{\rm *}(s,a_{\rm Q})] = & \mathbb E_s[\hat Q(s,a_{Q})-\hat Q(s,a_{Q})+Q^{\rm *}(s,a_{\rm Q})]\\
\leq & E_s[\hat Q(s,a_{Q})]+|\Lambda(s,a_{Q})|\\
\leq & E_s[\hat Q(s,a_{Q})]+\epsilon_{in}+(\epsilon_{out}-\epsilon_{in})\eta_{Q}.
\end{align}

Therefore
\[
\mathbb E_s[Q^{\rm true}(s,a_{\rm LRT})]
\ \ge\
\mathbb E_s[\hat Q(s,a_{\rm LRT})]-\varepsilon_{\rm in}-\nu\,\eta(\pi_{\rm LRT}),
\]
\[
\mathbb E_s[Q^{\rm true}(s,a_Q)]
\ \le\
\mathbb E_s[\hat Q(s,a_Q)]+\varepsilon_{\rm in}+\nu\,\eta(\pi_Q).
\]
Subtract the two displays and rearrange to obtain the stated bound for
$\Delta_{Q^{\rm true}}$.
\begin{equation}
\label{eq:app-ineq}
\Delta_{Q^{true}}\geq \Delta_{\hat Q}-2\varepsilon_{\rm in} -\nu\,\left(\eta(\pi_Q)+\eta(\pi_{\rm LRT})\right)
\end{equation}
\end{proof}

\subsection{Proof of Proposition~\ref{prop:eta}}\label{app:proof5.7}
\begin{proof}
Let \(\mathcal{E}\) be the event that the generated action is out-of-distribution, i.e., \(\mathcal{E} = \{a_0 \notin \mathcal{S}(s)\}\).
Let \(\mathcal{R}\) be the event that the trajectory-level likelihood-ratio test rejects the null hypothesis \(H_0\) (i.e., the gate allows deviation from the background), defined as \(\mathcal{R} = \{\ell_{\mathrm{cum}} \ge \tau\}\).

By the Law of Total Probability:
\[
\eta(\pi_{\rm LRT}) = \Pr(\mathcal{E}) = \Pr(\mathcal{E} \mid \mathcal{R})\Pr(\mathcal{R}) + \Pr(\mathcal{E} \mid \neg\mathcal{R})\Pr(\neg\mathcal{R}).
\]
\begin{enumerate}
    \item \textbf{Rejection Region (\(\mathcal{R}\)):} By the calibration procedure (Algorithm~\ref{alg:calib} and Theorem~\ref{thm:dkw}), the threshold \(\tau\) ensures \(\Pr(\mathcal{R}) \approx \alpha\). Since \(\Pr(\mathcal{E} \mid \mathcal{R}) \le 1\), the first term is bounded by \(\alpha\).
    \item \textbf{Acceptance Region (\(\neg\mathcal{R}\)):} The \textit{Monotone Support Assumption} states that if the test does not reject (gates remain mostly closed), the background head keeps the sample within support, i.e., \(\Pr(\mathcal{E} \mid \neg\mathcal{R}) \approx 0\).
\end{enumerate}
Combining these, \(\eta(\pi_{\rm LRT}) \lesssim 1 \cdot \alpha + 0 = \alpha\).

\textbf{Remark (Union bound):}  
If instead one bounds OOD using per-step Type-I rate $\alpha$ and assumes independence (or uses union bound), the trajectory-level OOD can be conservatively upper-bounded as
\[
\eta(\pi_{\rm LRT}) \le \sum_{t=1}^{T} \Pr(\text{step } t \text{ false activation}) \le \alpha T.
\]
This bound is much more conservative than the calibrated $\alpha$ used in practice.

\end{proof}

\subsection{Condition for LRT to Dominate}

\begin{corollary}[Feasible $\alpha$ via the bound]
\label{cor:alpha-union}
Under Assumption~\ref{prop:eta} with $c(\alpha)\le \alpha$,
a sufficient condition for $\Delta_{\text{true}}\ge 0$ is
\begin{equation}
\label{eq:alpha-max}
\alpha ~\le~ \alpha_{\max}
:= \frac{\Delta_{\widehat Q}-2\varepsilon_{\rm in}-\Delta_\varepsilon\,\eta(\pi_Q)}{\Delta_\varepsilon\,}.
\end{equation}
In particular, the feasible set $\{\alpha:\Delta_{\text{true}}\ge 0\}$ is nonempty whenever
$\Delta_{\widehat Q}>2\varepsilon_{\rm in}+\Delta_\varepsilon\,\eta(\pi_Q)$.
\end{corollary}

\begin{proof}
Plug $\eta(\pi_{\rm LRT})\le \alpha$ into \eqref{eq:app-ineq} and solve
$\Delta_{\text{true}}\ge 0$ for $\alpha$.
\end{proof}

\paragraph{How to widen the feasible interval.}
From \eqref{eq:alpha-max}, the numerator increases if we
(i) enlarge $\Delta_{\widehat Q}$ (better critic or stronger evidence-tied guidance),
and/or (ii) reduce $\eta(\pi_Q)$ (conservative anchoring/clipping for Q-steps).
Moreover, taking smaller $\alpha$ or smaller $\beta_{\max}$ further reduces $\eta(\pi_{\rm LRT})$.

\subsection{A tighter, Data-driven Alternative}
\label{app:alpha-empirical}
The surrogate $\eta(\pi_{\rm LRT})\le \alpha$ is often loose. Calibrate once and
\emph{measure} $\widehat\eta_{\rm LRT}(\alpha)$ on held-out states (our kNN proxy).
Then the same derivation yields the data-driven bound
\begin{equation}
\label{eq:alpha-emp}
\Delta_{\text{true}}
~\ge~
\Delta_{\widehat Q}
~-~2\varepsilon_{\rm in}
~-~\Delta_\varepsilon\big(\widehat\eta_{\rm LRT}(\alpha)+\widehat\eta(\pi_Q)\big),
\end{equation}
so any $\alpha$ satisfying
$\Delta_{\widehat Q}\ge 2\varepsilon_{\rm in}+\Delta_\varepsilon\big(\widehat\eta_{\rm LRT}(\alpha)+\widehat\eta(\pi_Q)\big)$
is sufficient. DKW-type concentration gives a uniform band for the empirical OOD curve,
translating into a high-probability guarantee for \eqref{eq:alpha-emp}.

\subsection{Optional Refinement with Consistency or Distance-based Bounds}
\label{app:alpha-consistency}
If the critic is consistent on support, e.g., with probability $1-\delta_n$,
$\varepsilon_{\rm in}\le r_n\!\to\!0$ as $n\to\infty$, then
\[
\Delta_{\text{true}}
~\gtrsim~
\Delta_{\widehat Q}
~-~2r_n
~-~\Delta_\varepsilon\big(\eta(\pi_{\rm LRT})+\eta(\pi_Q)\big),
\]
loosening the requirement on $\Delta_{\widehat Q}$ as $n$ grows.
Alternatively, if there exists $L_Q$ such that
$|\widehat Q-Q^{\text{true}}|\le \varepsilon_{\rm in}+L_Q\,\mathrm{dist}(a,\mathcal S(s))$,
then
\[
\Delta_{\text{true}}
~\ge~
\Delta_{\widehat Q}
-2\varepsilon_{\rm in}
- L_Q\,\mathbb E_s\!\big[\mathrm{dist}(a_{\rm LRT},\mathcal S(s))+\mathrm{dist}(a_Q,\mathcal S(s))\big].
\]
Using our per-step movement bound and the monotonic effect of $\alpha,\beta_{\max}$ on the
state-conditional distance, the last expectation can be turned into an explicit function of
$(\alpha,\beta_{\max})$, yielding a more problem-adapted feasible range for $\alpha$.
\qedhere

\subsection{Further Discussion for LRT dominate Q}
\label{app:discussion-dominate}
Combining Prop.~\ref{prop:gap} and Prop.~\ref{prop:eta} yields the sufficient condition
\begin{align*}
\Delta_{Q^{\rm true}}
~\ge~
\Delta_{\hat Q}
~-~ 2\varepsilon_{\rm in}
~-~ \nu\big(\alpha + \eta(\pi_Q)\big),
\\
\alpha_{\max}
~=~
\frac{\Delta_{\hat Q} - 2\varepsilon_{\rm in} - \nu\,\eta(\pi_Q)}{\nu},
\label{eq:sufficient-alpha}
\end{align*}
where

$$\Delta_{Q^{\rm true}}
:= \mathbb{E}_s\!\left[ Q^{\rm true}(s,a_{\rm LRT}) - Q^{\rm true}(s,a_Q) \right],$$
and
$$\Delta_{\hat Q}
:= \mathbb{E}_s\!\left[ \hat Q(s,a_{\rm LRT}) - \hat Q(s,a_Q) \right].$$
The bound decomposes into three interpretable pieces:

\begin{itemize}
\item \textbf{Predicted--value gap \(\Delta_{\hat Q}\).} How much the critic \(\hat Q\) prefers LRT samples over Q--guided samples on average. If the critic systematically overestimates off--support actions, \(\Delta_{\hat Q}\) can be small or negative; the condition is \emph{sufficient only} and may fail even when LRT is better in the environment.

\item \textbf{On--support error \(2\varepsilon_{\rm in}\).} A safety margin independent of \(\alpha\).
Better--fit (or pessimistic/trimmed) critics reduce this term and enlarge the feasible range.

\item \textbf{OOD penalty \(\nu(\alpha + \eta(\pi_Q))\).} The LRT part scales with \(\alpha\) via the calibrated bound \(\eta(\pi_{\rm LRT}) \lesssim \alpha\); the Q--guided baseline pays \(\nu\,\eta(\pi_Q)\).
Smaller \(\nu\)---e.g., via a smaller \(\beta_{\max}\) or late--step gating when \(\sigma_t\) is small---makes the condition easier to satisfy.
\end{itemize}

\paragraph{\textbf{When does LRT “dominate” Q--guidance?}}
If \(\alpha \le \alpha_{\max}\), then
\(\mathbb{E}_s[Q^{\rm true}(s,a_{\rm LRT})] \ge \mathbb{E}_s[Q^{\rm true}(s,a_Q)]\) is certified.
Two practical corollaries are useful:

\begin{itemize}
\item \textbf{A small--OOD Q baseline helps the certificate.}
If \(\eta(\pi_Q)\) is small (Q--guidance rarely leaves support), the denominator in \(\alpha_{\max}\) is fixed while the numerator increases, so \(\alpha_{\max}\) grows:
\(\eta(\pi_Q)\!\downarrow \Rightarrow \alpha_{\max}\!\uparrow\).
In the extreme \(\eta(\pi_Q){=}0\), the feasible range reads
\(\alpha \le (\Delta_{\hat Q} - 2\varepsilon_{\rm in})/\nu \).
Intuitively, when Q--guidance already behaves conservatively, we can open the LRT gate more (choose a larger \(\alpha\)) and still retain a certificate.

\item \textbf{Levers that enlarge \(\alpha_{\max}\).}
(i) \emph{Lower \(\nu\)} (e.g., smaller \(\beta_{\max}\) or restricting gating to small--variance steps; anchoring gradients at \(\mu_u\) also lowers off--support effects);
(ii) \emph{Increase \(\Delta_{\hat Q}\)} by taking a modest Q--step near \(\mu_{\rm LRT}\) when the gate opens or by using evidence--tied/annealed guidance.
\end{itemize}

\paragraph{Why we still sweep \(\alpha\) in practice.}
The certificate is conservative for two reasons:
(i) it uses the calibrated bound \(\eta(\pi_{\rm LRT})\!\lesssim\!\alpha\);
(ii) \(\Delta_{\hat Q}\) can be biased when \(\hat Q\) overestimates OOD.
Thus even when \(\alpha > \alpha_{\max}\) and the inequality cannot be certified, LRT may still outperform Q--guidance empirically.
In experiments we therefore \emph{sweep \(\alpha\)} on a logarithmic grid and select the knee of the return--risk curve (Fig.~ \ref{pic:alba}); the bound serves as a sanity upper bound for \(\alpha\), not as a hard constraint.

\paragraph{\textbf{A practical process.}}
\begin{enumerate}
\item On a small held--out state batch, estimate \(\widehat{\Delta}_{\hat Q}\) and \(\widehat{\eta}(\pi_Q)\) by generating actions from both samplers and evaluating \(\hat Q\) and the state--conditional OOD rate.
\item Form the diagnostic
\[
\widehat{\alpha}_{\max}
~=~
\frac{\widehat{\Delta}_{\hat Q} - 2\varepsilon_{\rm in} - \nu\,\widehat{\eta}(\pi_Q)}{\nu}
\quad
\](use a conservative constant if \(\varepsilon_{\rm in}\) is unknown).
\item Sweep \(\alpha\) on a log grid in 
\([\,\min(0.5\,\widehat{\alpha}_{\max},\,0.2),\,0.2\,]\),
and select the knee on the return--risk curves using the \emph{calibrated} LRT sampler.
\end{enumerate}

\section{Appending with Q-guidance}
\subsection{Choice of the evaluation center $a_c$.}\label{app:evaluation-center}
\begin{itemize}
\item \textbf{$a_c=\mu_u$ (conservative).} Evaluate $\nabla_a \hat Q$ near the behavior-supported mean; most stable and robust off-support.
\item \textbf{$a_c=\mu_{\text{LRT},t}$ (exploitative).} When evidence opens the gate ($\beta_t$ large), probe gradients closer to the ``good'' direction.
\item \textbf{$a_c=(1-\rho)\mu_u+\rho\,\mu_{\text{LRT},t}$ (interpolate).} $\rho\!\in[0,1]$ trades conservatism for exploitation; a simple adaptive choice is $\rho=\beta_t/\beta_{\max}$.
\end{itemize}

\section{Implementation Details}
\label{app:exp-impl}
\textbf{Training.} IQL critic: MLPs (two hidden layers, 256 units), $\gamma\!=\!0.99$, expectile $0.7$ (following IQL~\cite{Kostrikov2022IQL}), Adam ($3{\times}10^{-5}$), batch 1024, $30$ epochs. Advantages are computed on standardized $(s,a)$. Labels: global top‑$p$ with $p\!=\!0.2$ unless stated. Diffusion: $T\!=\!50$, DDPM linear schedule on the forward noise $(1-\alpha_t)$ with endpoints
$(1-\alpha_1\!=\!10^{-4},\, 1-\alpha_T\!=\!2{\times}10^{-2})$,
, backbone MLP with SiLU activations, two $\epsilon$‑heads; AdamW ($2{\times}10^{-4}$), batch 1024, 150 epochs. Class‑balancing with an EMA of the positive rate; optional soft advantage weight within positives with temperature $\tau_A$ and cap $u_{\max}$.

\textbf{Inference and calibration.} We use the \emph{posterior} variance $\tilde\sigma_t^2$ (DDPM reverse variance) for both heads and LLR; alternative variance parameterizations have also been explored~\citep{Nichol2021ImprovedDDPM}. Gate defaults: $\beta_{\max}\!=\!1$, $\delta\!\in\![1,2]$ (ablate). We calibrate $\hat\tau$ on a held‑out state set (disjoint from any rollout episodes) with a fixed‑point update
\[
\tau \leftarrow \mathrm{Quantile}_{1-\alpha}\big\{\ell_{\mathrm{cum}}(\tau)\big\},
\]
running $K$ iterations (default $K\!=\!6$) with light momentum to avoid oscillation, and then \emph{reuse} the resulting $\hat\tau$ for all deployments with the same $(\beta_{\max},\delta)$ and (if enabled) the same $Q$‑composition.

\textbf{Evaluation protocol.}
\label{sec:exp-protocol}
Unless stated, each configuration is evaluated over $10$ random seeds $\times$ $10$ episodes per seed. We report means $\pm$ standard deviation across all seeds. Actions are mapped back to the environment scale and clipped to bounds before stepping. 

\section{Additional Experimental Discussion}
\label{app:exp-disc}

\subsection{Interpreting Table~\ref{tab:main}.}
On all three \texttt{*-medium-replay-v2} tasks, \textbf{LRT} consistently yields the \emph{lowest} state-conditional OOD (reported in Table~\ref{tab:main} in units of $\times 10^{-2}$), while keeping returns competitive. For \emph{Hopper-replay}, LRT attains $329\pm22$ return at $1.84\pm0.15$ OOD ($\times 10^{-2}$), versus Q at $363\pm34$ and $6.32\pm0.53$. For \emph{Walker2d-replay}, LRT achieves $315\pm21$ at $0.41\pm0.07$, whereas Q and LRT+Q raise return to $\approx 3.7\times10^2$ but at $\approx 3.8$ OOD (i.e., $\approx 3.8\times10^{-2}$). On \emph{HalfCheetah-replay}, LRT remains conservative in OOD ($1.14\pm0.25$), while Q and LRT+Q incur much higher OOD ($\approx 13$, i.e., $\approx 13\times10^{-2}$) for modest return gains.

On the \texttt{*-medium-v2} tasks, adding a critic step (\textbf{LRT+Q}) provides the strongest return improvements, while LRT alone remains the conservative anchor. For \emph{Hopper-medium}, LRT+Q reaches the best return ($1197\pm47$) with higher OOD ($11.32\pm0.83$) than LRT ($744\pm71$, $9.21\pm0.46$). For \emph{HalfCheetah-medium}, both Q and LRT+Q substantially increase return over LRT (from $3526\pm38$ to $>4400$) with a moderate OOD increase (from $3.61\pm0.97$ to $\approx 5$). For \emph{Walker2d-medium}, Q and LRT+Q achieve much higher returns ($2282\pm178$ and $2448\pm196$) than LRT ($568\pm49$) while operating at lower OOD ($\approx 5$ vs.\ $10.79\pm2.51$; all $\times 10^{-2}$), indicating that value-driven exploitation dominates this dataset.

Overall, \textbf{LRT} serves as a \emph{low-risk anchor} that suppresses off-support behavior, and \textbf{LRT+Q} moves along the return--risk frontier when critic guidance is beneficial.

\subsection{Seed-to-seed Variability on \texttt{halfcheetah-medium-replay-v2}}
\label{app:exp-variability}

\paragraph{\textbf{Why is the std large, and why report relative gaps?}}
The std in Table~\ref{tab:main} is computed across random seeds (not episode-level s.e.). On \texttt{halfcheetah-medium-replay-v2}, absolute returns vary substantially across seeds for all methods, which inflates the across-seed std and can obscure method effects. To factor out shared seed difficulty, we additionally report per-seed \emph{relative gaps} against Q and summarize them with mean$\pm$std across seeds using a \emph{global} denominator,
$\Delta^{\mathrm{global}}_{\mathrm{rel}}=(R_{\mathrm{method}}-R_Q)/|\overline{R_Q}|$. Table~\ref{tab:relgap} shows that the direction of improvement is consistent for several datasets (e.g., LRT+Q yields modest positive gaps on Hopper and HalfCheetah \texttt{medium}), whereas on \texttt{halfcheetah-medium-replay-v2} the mean gaps are small relative to their variability (LRT: $-6.72\%\pm22.05\%$, LRT+Q: $+2.85\%\pm14.67\%$), indicating that method differences remain \emph{seed-sensitive} and may flip sign across seeds (and hyperparameter choices). Thus, the large std is not solely a reporting artifact; it reflects genuine sensitivity of the comparative performance on replay datasets.

\begin{table}[H]
\centering
\small
\caption{\textbf{Global-denominator relative gaps vs.\ Q (mean$\pm$std across seeds).}
We report $\Delta^{\mathrm{global}}_{\mathrm{rel}}=(R_{\mathrm{method}}-R_Q)/|\overline{R_Q}|$,
where $\overline{R_Q}$ is the mean return of the corresponding Q baseline (same $\lambda$ for LRT+Q).
Numbers are in \%.}
\label{tab:relgap}
\begin{tabular}{lcc}
\toprule
Dataset & LRT vs Q (\%) & LRT+Q vs Q (\%) \\
\midrule
hopper-medium-replay-v2      & $-8.52\pm 12.63$ & $+3.01\pm 9.33$ \\
halfcheetah-medium-replay-v2 & $-6.72\pm 22.05$ & $+2.85\pm 14.67$ \\
walker2d-medium-replay-v2    & $-16.36\pm 11.79$ & $+0.25\pm 4.21$ \\
hopper-medium-v2             & $-36.73\pm 7.56$  & $+2.78\pm 4.02$ \\
halfcheetah-medium-v2        & $-19.93\pm 1.12$  & $+1.51\pm 1.08$ \\
walker2d-medium-v2           & $-75.11\pm 7.93$  & $+7.30\pm 7.09$ \\
\bottomrule
\end{tabular}
\end{table}

\subsection{Case: \texttt{walker2d-medium-v2}}
\label{app:walker2d-exception}
\paragraph{\textbf{Interpreting LRT vs.\ Q on Walker2d-medium.}}
In Table~\ref{tab:main}, the reported OOD is in units of $\times 10^{-2}$. For \texttt{walker2d-medium-v2}, the return-optimal setting for LRT may correspond to a relatively large $\alpha$, which increases gate activations and can raise OOD, yet still falls short of the exploitation provided by explicit value gradients. This highlights a key point: LRT controls a calibrated activation risk, while the critic step primarily drives return; when value guidance is reliable and highly beneficial, LRT alone can be under-exploitative, whereas LRT+Q moves along the return--risk frontier.

\section{Ablation}
\subsection{Sensitivity to the ``Good Data'' Fraction $p$}
\label{app:p-ablation}

In our method, $p$ is only used once to mark the top-$p$ transitions as ``good'' when training the conditional head.
We set $p = 0.2$ as a simple default that balances dataset coverage while emphasizing high-advantage samples.

We ran an ablation on \texttt{hopper-medium-replay-v2} with the same training setup.
The averaged results are shown in Table~\ref{tab:p_ablation}.

\begin{table}[h]

\begin{tabular}{c|cc|cc}
\toprule
$p$ & Return (mean) & Return (std) & OOD (mean) &  OOD (std) \\
\hline
0.05 & 344.1 & 12.0 & 0.0172 & 0.0010 \\
0.10 & 306.2 & 11.8 & 0.0193 & 0.0029 \\
0.20 & 342.4 & 22.9 & 0.0189 & 0.0020 \\
0.30 & 307.1 & 12.6 & 0.0168 & 0.0018 \\
\bottomrule
\end{tabular}
\caption{Ablation over the ``good data'' fraction $p$:
$p \in \{0.05, 0.10, 0.20, 0.30\}$, 
$\alpha \in \{0.10, 0.05, 0.02, 0.005\}$}
\label{tab:p_ablation}
\end{table}

The \textbf{returns and state-conditional OOD rates are broadly stable across $p$}.
While $p=0.2$ is slightly better, other choices also perform comparably,
indicating that our method is not overly sensitive to the exact choice of $p$.

\subsection{Stress Test under Variance Mismatch}\label{app:var-stress}
We conducted a \textbf{stress test on \texttt{hopper-medium-replay-v2}}:
we deliberately \textbf{violate the equal-variance assumption at sampling time}
by inflating the background head’s noise standard deviation by a factor
$s \in \{1,2,3\}$, while keeping the calibration procedure (which assumes equal variances) unchanged, see Table~\ref{tab:var_stress}.

As expected, returns drop sharply, Type-I error inflates, and OOD increases as $s$ grows.
At scale $s=1$, risk control holds; at $s=2,3$, both Type-I error and OOD increase,
leading to large return degradation.
This confirms that the guarantee is model-conditional and degrades under deliberate mis-specification.
This experiment reflects a stress test rather than a realistic configuration.

\begin{table}[H]
\centering
\begin{tabular}{cccccc}
\toprule
scale & $\alpha$ & $\tau$ & Return & Type-I & OOD \\
\midrule
1 & 0.10 & 5.5769 & 361.46 & 0.0005 & 0.0202 \\
1 & 0.05 & 6.7092 & 329.43 & 0.0010 & 0.0154 \\
1 & 0.01 & 9.8676 & 316.82 & 0.0000 & 0.0204 \\
\midrule
2 & 0.10 & 5.5769 & 127.75 & 0.0560 & 0.2023 \\
2 & 0.05 & 6.7092 & 137.83 & 0.0380 & 0.1871 \\
2 & 0.01 & 9.8676 & 144.92 & 0.0085 & 0.1945 \\
\midrule
3 & 0.10 & 5.5769 & 70.81 & 0.1490 & 0.3362 \\
3 & 0.05 & 6.7092 & 62.72 & 0.1030 & 0.3437 \\
3 & 0.01 & 9.8676 & 55.16 & 0.0455 & 0.3462 \\
\bottomrule
\end{tabular}
\caption{Stress test under background variance scaling.}
\label{tab:var_stress}
\end{table}

\end{document}